\def\isarxiv{1}

\ifdefined\isarxiv

\documentclass[12pt]{article}
\usepackage[margin=1.25in]{geometry}
\usepackage{times}
\usepackage{hyperref}
\else
\documentclass{article}
\usepackage{multicol,enumitem}
\usepackage{hyperref}
\usepackage{icml2024}
\fi

\usepackage{natbib}

\usepackage{hyperref}       % hyperlinks
\usepackage{url}   

\usepackage{amsmath}
\usepackage{amssymb}
\usepackage{amsthm}

\usepackage{booktabs}       % professional-quality tables
\usepackage{amsfonts}       % blackboard math symbols
\usepackage{nicefrac}       % compact symbols for 1/2, etc.
\usepackage{microtype}      % microtypography
\usepackage{xcolor}         % colors
\usepackage{amsmath,amsthm,amssymb}
\usepackage{graphicx}
\usepackage{cleveref}
\usepackage{comment}
\usepackage{bbm}
\usepackage{textcomp}
\usepackage{wrapfig}
\usepackage{float}

\usepackage{wrapfig}

\usepackage{url}            % simple URL typesetting
\usepackage{booktabs}       % professional-quality tables
\usepackage{amsfonts}       % blackboard math symbols
\usepackage{nicefrac}       % compact symbols for 1/2, etc.
\usepackage{microtype}      % microtypography
\usepackage{multirow}
\usepackage{algorithm}
\usepackage{algorithmic}
\usepackage{amsmath}
\usepackage{amsthm}
\usepackage{amssymb}
\usepackage{color}
\usepackage[english]{babel}
\usepackage{graphicx}
\usepackage{grffile}
\usepackage{wrapfig,epsfig}
\usepackage{url}
\usepackage{color}
\usepackage[T1]{fontenc}
\usepackage{bbm}
\usepackage{comment}
\usepackage{tikz}
\usepackage{colortbl}
\usepackage{pgfplots}
\pgfplotsset{compat=1.8}
\tikzset{elegant/.style={smooth,thick,samples=500,magenta}}

\theoremstyle{plain}
\newtheorem{theorem}{Theorem}[section]
\newtheorem{lemma}[theorem]{Lemma}
\newtheorem{remark}[theorem]{Remark}
\newtheorem{corollary}[theorem]{Corollary}

\theoremstyle{definition}
\newtheorem{definition}[theorem]{Definition}

\newtheorem{example}[theorem]{Example}
\newtheorem{problem}[theorem]{Problem}

\crefname{assumption}{Assumption}{Assumptions}

\theoremstyle{plain}
\newtheorem*{thm*}{Theorem}

\theoremstyle{plain}

\newcommand{\TV}{\texttt{TV}}
\newcommand{\cP}{\mathcal{P}}
\newcommand{\F}{\mathcal{F}}

\newcommand{\Binom}{\mathrm{Binom}}

\newcommand{\supp}{\mathrm{supp}}

\newcommand{\ump}{\mathrm{ump}}
\newcommand{\minmax}{\mathrm{minmax}}

\newcommand{\R}{\mathbb{R}}
\newcommand{\E}{\mathbb{E}}
\newcommand{\PP}{\mathbb{P}}

\newcommand{\cS}{\mathcal{S}}

\newcommand{\wt}{\widetilde}

\definecolor{b2}{RGB}{51,153,255}
\definecolor{mygreen}{RGB}{80,180,0}

\title{Towards Optimal Statistical Watermarking}

\ifdefined\isarxiv
\author{
}
\date{}

\else

\author{%
Baihe Huang \\ University of California, Berkeley\\{
  \texttt{baihe\_huang@berkeley.edu} }
  \And
  Banghua Zhu \\ University of California, Berkeley\\{
  \texttt{banghua@berkeley.edu}}
  \And
  Hanlin Zhu \\ University of California, Berkeley\\{
  \texttt{hanlinzhu@berkeley.edu}}
  \And
  Jason D. Lee \\ Princeton University\\{
  \texttt{jasonlee@princeton.edu } }
  \And
  Jiantao Jiao \\ University of California, Berkeley\\{
  \texttt{jiantao@eecs.berkeley.edu } }
  \And
  Michael I. Jordan \\ University of California, Berkeley\\{
  \texttt{jordan@cs.berkeley.edu} }
}

\fi

\ifdefined\usebigfont

\usepackage{times}
\usepackage[fontsize=13pt]{scrextend}
\makeatletter
\@ifpackageloaded{geometry}{\AtBeginDocument{\newgeometry{letterpaper,left=1.56in,right=1.56in,top=1.71in,bottom=1.77in}}}{\usepackage[letterpaper,left=1.56in,right=1.56in,top=1.71in,bottom=1.77in]{geometry}}
\AtBeginDocument{\newgeometry{letterpaper,left=1.56in,right=1.56in,top=1.71in,bottom=1.77in}}
\makeatother
\else
\fi
\begin{document}

\setlength{\abovedisplayskip}{3.2pt}
\setlength{\belowdisplayskip}{3.2pt}

\ifdefined\isarxiv
\author{
Baihe Huang\thanks{
  \texttt{baihe\_huang@berkeley.edu}. University of California, Berkeley. }
  % examples of more authors
  \and
  Hanlin Zhu\thanks{
  \texttt{hanlinzhu@berkeley.edu}. University of California, Berkeley. }
  \and
  Banghua Zhu\thanks{
  \texttt{banghua@berkeley.edu}. University of California, Berkeley. }
  \and
  Kannan Ramchandran
  \thanks{
  \texttt{kannanr@berkeley.edu}. University of California, Berkeley. }
  \and
  Michael I. Jordan\thanks{
  \texttt{jordan@cs.berkeley.edu}. University of California, Berkeley. }
  \and
  Jason D. Lee\thanks{
  \texttt{jasonlee@princeton.edu}. Princeton University. }
  \and
  Jiantao Jiao\thanks{
  \texttt{jiantao@eecs.berkeley.edu}. University of California, Berkeley. }
}
%\begin{titlepage}
  \maketitle
  \begin{abstract}
  We study statistical watermarking by formulating it as a hypothesis testing problem, a general framework which subsumes all previous statistical watermarking methods. Key to our formulation is a coupling of the output tokens and the rejection region, realized by pseudo-random generators in practice, that allows non-trivial trade-offs between the Type I error and Type II error. We characterize the Uniformly Most Powerful (UMP) watermark in the general hypothesis testing setting and the minimax Type II error in the model-agnostic setting. In the common scenario where the output is a sequence of $n$ tokens, we establish nearly matching upper and lower bounds on the number of i.i.d. tokens required to guarantee small Type I and Type II errors. Our rate of $\Theta(h^{-1} \log (1/h))$ with respect to the average entropy per token $h$ highlights potentials for improvement from the rate of $h^{-2}$ in the previous works. Moreover, we formulate the robust watermarking problem where the user is allowed to perform a class of perturbations on the generated texts, and characterize the optimal Type II error of robust UMP tests via a linear programming problem. To the best of our knowledge, this is the first systematic statistical treatment on the watermarking problem with near-optimal rates in the i.i.d. setting, which might be of interest for future works.
  \end{abstract}
 % \thispagestyle{empty}
%\end{titlepage}

%\pagenumbering{roman}
%{\small
%\hypersetup{linkcolor=black}
%\tableofcontents
%}
%\newpage
\else

\twocolumn[
\icmltitle{Towards Optimal Statistical Watermarking}



\icmlsetsymbol{equal}{*}

\begin{icmlauthorlist}
\icmlauthor{Firstname1 Lastname1}{equal,yyy}
\icmlauthor{Firstname2 Lastname2}{equal,yyy,comp}
\icmlauthor{Firstname3 Lastname3}{comp}
\icmlauthor{Firstname4 Lastname4}{sch}
\icmlauthor{Firstname5 Lastname5}{yyy}
\icmlauthor{Firstname6 Lastname6}{sch,yyy,comp}
\icmlauthor{Firstname7 Lastname7}{comp}
\icmlauthor{Firstname8 Lastname8}{sch}
\icmlauthor{Firstname8 Lastname8}{yyy,comp}
\end{icmlauthorlist}

\icmlaffiliation{yyy}{Department of XXX, University of YYY, Location, Country}
\icmlaffiliation{comp}{Company Name, Location, Country}
\icmlaffiliation{sch}{School of ZZZ, Institute of WWW, Location, Country}

\icmlcorrespondingauthor{Firstname1 Lastname1}{first1.last1@xxx.edu}
\icmlcorrespondingauthor{Firstname2 Lastname2}{first2.last2@www.uk}

\icmlkeywords{Machine Learning, ICML}

\vskip 0.3in
]



\printAffiliationsAndNotice{\icmlEqualContribution} % otherwise use the standard text.

% \maketitle
\begin{abstract}
We study statistical watermarking by formulating it as a hypothesis testing problem, a general framework which subsumes all previous statistical watermarking methods. Key to our formulation is a coupling of the output tokens and the rejection region, realized by pseudo-random generators in practice, that allows non-trivial trade-offs between the Type I error and Type II error. We characterize the Uniformly Most Powerful (UMP) watermark in the general hypothesis testing setting and the minimax Type II error in the model-agnostic setting. In the common scenario where the output is a sequence of $n$ tokens, we establish nearly matching upper and lower bounds on the number of i.i.d. tokens required to guarantee small Type I and Type II errors. Our rate of $\Theta(h^{-1} \log (1/h))$ with respect to the average entropy per token $h$ highlights a fundamental gap between the rate of $O(h^{-2})$ in the previous works. 
% For scenarios where the detector lacks knowledge of the model's distribution, we introduce the concept of model-agnostic watermarking and establish the minimax bounds for the resultant increase in Type II error. 
Moreover, we formulate the robust watermarking problem where the user is allowed to perform a class of perturbation on the generated texts, and characterize the optimal type II error of robust UMP tests via a linear programming problem. To the best of our knowledge, this is the first systematic statistical treatment on the watermarking problem with near-optimal rates in the i.i.d. setting, and might be of interest for future works.
\end{abstract}

\fi

\section{Introduction}

The prevalence of large language models (LLMs) in recent years makes it challenging and important to detect whether a human-like text is produced by the LLM system~\citep{kirchenbauer2023watermark,kuditipudi2023robust,christ2023undetectable, yoo2023advancing,fernandez2023three,fu2023watermarking, wang2023towards,yang2023towards,liu2023private,zhao2023provable,koike2023outfox}. 
On the one hand, some of the most advanced LLMs to date, such as GPT-4~\citep{openai2023gpt4}, are good at producing human-like texts, which might be hard to distinguish from human-generated texts even for humans in various scenarios. On the other hand, it is important to keep human-produced text datasets separated from machine-produced texts in order to avoid the spread of misleading information~\citep{vincent2022ai} and the contamination of training datasets for future language models~\citep{kuditipudi2023robust}. 

To detect machine-generated content, a recent line of work~\citep{kirchenbauer2023watermark,kuditipudi2023robust,christ2023undetectable} proposes to inject \emph{statistical watermarks}, a signal embedded within the generated texts which reveals the generation source, into texts. As discussed in \citet{kuditipudi2023robust}, there are three desirable properties of watermarking: 1. distortion-free: the watermark should not alert the distribution of the generated texts; 2. agnostic: the detector should not 
%needs not to 
know the language model or the prompt; 3. robust: the detector should be able to detect the watermark even under slight perturbation of the generated texts. 
However, previously proposed methods are either heuristic or guaranteed by different, sub-optimal mathematical descriptions of the above properties, making it difficult to systematically evaluate the watermarking schemes and to draw useful statistical conclusions. 

Motivated by this, we propose a unifying formulation of statistical watermarking based on hypothesis testing, and study the trade-off between the Type I error and the Type II error. 
More specifically, our contributions are summarized as follows:
\begin{itemize}%[itemsep=6pt]
    \item We formulate statistical watermarking as a hypothesis testing problem with a random rejection region, and specify model-agnostic watermarking, where the distribution of the rejection region is independent of the underlying model distribution, as a notion highly practical in real-world applications.
    \item 
    We find the optimal Type II error among all level-$\alpha$ tests and explicitly characterize the most powerful watermarking scheme that achieves it. 
    For model-agnostic watermarking, we construct the optimal distribution of the reject region and establish the minimax increase in Type II error in comparison to the most powerful watermarking schemes.
    \item In the context where the sample is a sequence of many i.i.d. tokens, we provide nearly-matching upper and lower bounds for the minimum number of tokens required to guarantee small type I and type II errors. Our rate $h^{-1}\log (1/h)$ improves upon previous works featuring a rate of $h^{-2}$, in terms of $h$ --- the average entropy of per generated tokens.
    \item Additionally, we formulate a robust watermarking problem where the watermarking scheme is robust to a class of perturbations that the user can employ to the outputs. In this setting, we also characterize the optimal type II error and the construction of the robust watermarking scheme via a linear program.
\end{itemize}

% A discussion of more related works is deferred to \Cref{sec:related_work}

\subsection{Related works}
% \section{Related Works}
\label{sec:related_work}
\ifdefined\isarxiv
% \else
Watermarking is a powerful white-box method for detecting LLM-generated texts~\citep{tang2023science}. 
Watermarks can be injected either into a pre-existing text (edit-based watermarks) or during the text generation (generative watermarks). Our work falls in the latter category. Edit-based watermarking~\citep{rizzo2019fine,abdelnabi2021adversarial,yang2022tracing,kamaruddin2018review} has been the focus of several studies in the past. The concept of generative watermarking dates back to the work of~\citet{
venugopal-etal-2011-watermarking}, while our work is more relevant to a recent line of works~\citep{aaronson2022my,kirchenbauer2023watermark,kuditipudi2023robust,christ2023undetectable} that introduce statistical signals into text generation. Specifically, \citet{kirchenbauer2023watermark} increases the probability that tokens are chosen from a randomly sampled `green' list; \citet{aaronson2022my} selects the token $i$ that maximizes keys randomly sampled from exponential distributions with mean $1/p_i$; \citet{christ2023undetectable} samples the tokens by solving the optimal transport from uniform distribution in $[0,1]$; \citet{kuditipudi2023robust} introduces inverse transform sampling as a distortion-free watermarking method; \citet{zhao2023provable}  proposes a simplified variation of \citet{kirchenbauer2023watermark} where a fixed Green-Red split is used consistently. These watermarks are evaluated in the benchmark of~\citet{piet2023mark}.

Statistical watermarking techniques share the similarity that the outputs are correlated with some secret keys (which could come from either external randomness or internal hashing), thereby coupling the rejection region and the outputs in the hypothesis testing. This fact is recognized by recent works of~\citet{kuditipudi2023robust,zhao2023provable}, where model-agnosticism in the detection phase is also emphasized. The exponential scheme in~\citet{aaronson2022watermarking}, the inverse transform sampling scheme in~\citet{kuditipudi2023robust}, and the binary scheme in~\citet{christ2023undetectable} come with theoretical guarantees that (i) the watermarked model distribution cannot be distinguished from the original distribution (called undetectability~\citep{christ2023undetectable} or distortion-freeness~\citep{kuditipudi2023robust}), and (ii) the outputs from the watermarked models are statistically detectable as long as the entropy is lower bounded. In contrast, \citet{kirchenbauer2023watermark} is not distortion-free, nonetheless enjoying little degradation in generation quality and provable detectability~\citep{zhao2023provable} with suitable parameter choice of the bias parameter (logits increase $\delta$ in the `green' list). Despite the aforementioned theoretical efforts in establishing guarantees for existing watermarks, the fundamental tradeoff in this hypothesis testing problem and the rates on the required number of generated tokens remain unsolved.  

Watermarks can also be injected with private forgeability and public verifiability~\citep{fairoze2023publicly}, hence functioning effectively as digital signatures.  Meanwhile, various attack algorithms against watermarking schemes were also studied \citep{kirchenbauer2023watermark,kirchenbauer2023reliability,sato2023embarrassingly,zhang2023watermarks,kuditipudi2023robust}. These attacking schemes apply quality-preserving perturbations to the watermarked outputs in delicate ways, and are therefore modelled by the perturbation graph (Definition~\ref{def:perturb_graph}) in the robust watermark framework in Section~\ref{sec:robust}. With the success of various attacking methods, robustness becomes an important consideration in watermarking techniques. However, \citet{zhang2023watermarks} proves that it is only feasible to achieve robustness to a well-specified set of attacks, instead of all. This fact aligns with our Theorem~\ref{thm:robust_rates}, which characterizes the fundamental limits of robust watermarking under different attacking powers. 
% \citet{piet2023mark} creates a benchmark for current statistical watermarks.
\else
Watermarking is a powerful white-box method for detecting LLM-generated texts~\citep{tang2023science}. 
Watermarks can be injected either into a pre-existing text (edit-based watermarks) or during the text generation (generative watermarks). Our work falls in the latter category. Edit-based watermarking~\citep{rizzo2019fine,abdelnabi2021adversarial,yang2022tracing,kamaruddin2018review} has been the focus of several studies in the past. The concept of generative watermarking dates back to the work of~\citet{
venugopal-etal-2011-watermarking}, while our work is more relevant to a recent line of works~\citep{aaronson2022my,kirchenbauer2023watermark,kuditipudi2023robust,christ2023undetectable,zhao2023provable} that introduce statistical signals into text generation. Specifically, \citet{kirchenbauer2023watermark} increases the probability that tokens are chosen from a randomly sampled `green' list; \citet{aaronson2022my} selects the token $i$ that maximizes keys randomly sampled from exponential distributions with mean $1/p_i$; \citet{christ2023undetectable} samples the tokens by solving the optimal transport from uniform distribution in $[0,1]$; \citet{kuditipudi2023robust} introduces inverse transform sampling as a distortion-free watermarking method; \citet{zhao2023provable}  proposes a simplified variation of \citet{kirchenbauer2023watermark} where a fixed Green-Red split is used consistently. These watermarks are evaluated in the benchmark of~\citet{piet2023mark}. More discussions on related works are deferred to Appendix~\ref{sec:add_related_works}.

% Statistical watermarking techniques share the similarity that the outputs are correlated with some secret keys (which could come from either external randomness or internal hashing), thereby coupling the rejection region and the outputs in the hypothesis testing. This fact is recognized by recent works of~\citet{kuditipudi2023robust,zhao2023provable}, where model-agnosticity in the detection phase is also emphasized. The exponential scheme in~\citet{aaronson2022watermarking}, inverse transform sampleing scheme in~\citet{kuditipudi2023robust}, and binary scheme in~\citet{christ2023undetectable} come with theoretical guarantees that (i) the watermarked model distribution cannot be distinguished from the original distribution (called undetectability~\citep{christ2023undetectable} or distortion-freeness~\citep{kuditipudi2023robust}), and (ii) the outputs from watermarked models are statistically detectable as long as the entropy is lower bounded. In contrast, \citet{kirchenbauer2023watermark} is not distortion-free, nonetheless enjoying little degradation in generation quality and provable detectability~\citep{zhao2023provable} with suitable parameter choice of the bias parameter (logits increase $\delta$ in the `green' list). Despite the aforementioned theoretical efforts in establishing guarantees for existing watermarks, the fundamental tradeoff in this hypothesis testing problem and the rates on the required number of generated tokens remain unsolved.  
\fi
\subsection{Notation}

Define $(x)_+:= \max\{x,0\}$, $x \wedge y := \min\{x,y\}, x \vee y = \max\{x,y\}$. For any set $A$, we write $A^c$ as the complement of set $A$, $|A|$ as its cardinality, and $2^A:= \{B: B \subset A\}$ as the power set of $A$. We use notations $g(n) = O\left(f(n)\right)$, $g(n) = \Omega\left(f(n)\right)$, and $g(n) = \Theta\left(f(n)\right)$ to denote that there exists numerical constants $C_1,c_2,C_3,c_4$ such that for all $n > 0$: $g(n) \leq C_1 \cdot f(n)$, $g(n) \geq c_2 \cdot f(n)$, and $c_4 \cdot f(n) \leq g(n) \leq C_3 \cdot f(n)$, respectively. 
% Furthermore, we use $\wt{O},\wt{\Omega}$, and $\wt{\Theta}$ to suppress poly-logarithmic factors. 
Throughout, we use $\ln$ to denote natural logarithm. 

The total variation (TV) distance between two probability measures $\mu,\nu$ is denoted by $\TV(\mu\|\nu)$. We use $\supp(\mu)$ to denote the support of a probability measure $\rho$. Given a sample space $\Omega$, let $\Delta(\Omega)$ denote the set of all probability measures over $\Omega$ (take the discrete $\sigma$-algebra). We write $\delta_x$ as the Dirac measure on $x$, i.e., $\delta_x(A) = \begin{cases}
    1, &~ x \in A\\ 0, &~ x \notin A
\end{cases}$. A coupling for two distributions (i.e. probability measures) is a joint distribution of them.
\section{Watermarking as a Hypothesis Testing Problem}

In the problem of statistical watermarking, a service provider (e.g., a language model system), who possesses a distribution $\rho$ over a sample space $\Omega$, aims to make the samples from the service provider distinguishable by a detector, without changing $\rho$. The service provider achieves this by sharing a watermark key (generated from a distribution that is \emph{coupled with} $\rho$) with the detector, with the goal of controlling both the Type I error (an independent output is falsely detected as from $\rho$) and the Type II error (an output from $\rho$ fails to be detected). This random key together with the detection rule constitute a (random) rejection region. In the following, we formulate this problem as hypothesis testing with random rejection regions.

\begin{problem}[Watermarking]\label{prob:crypto_stat_watermark}
Fix $\epsilon \geq 0$. 
Given a probability measure $\rho$ over sample space $\Omega$\footnote{Throughout we will assume that $\Omega$ is discrete, as in most applications.}, an $\epsilon$-distorted watermarking scheme of $\rho$ is a probability measure $\cP$ (a joint probability of the output $X$ and the rejection region $R$) over the sample space $\Omega \otimes 2^{\Omega}$ such that $\TV(\cP(\cdot,2^{\Omega})\|\rho) \leq \epsilon$, where $\cP(\cdot,2^{\Omega})$ is the marginal probability of $X$ over $\Omega$. In the generation phase, the service provider samples $(X,R)$ from $\cP$, provides the output $X$ to the service user, and sends the rejection region $R$ to the detector.

In the detection phase, a detector is given a tuple $(X,R) \in \Omega \otimes 2^{\Omega}$ where $X$ is sampled from an unknown distribution and $R$, given by the service provider, is sampled from the marginal probability $\cP(\Omega,\cdot)$ over $2^\Omega$. The detector is tasked with using $R$ to conduct a hypothesis test that involves two competing hypotheses:
\begin{align*}
    H_0: &~ X \text{ is sampled independently from } R,\\
    \textit{versus}~~H_1: &~ (X,R) \text{ is sampled from the joint distribution } \cP.
\end{align*}
The \emph{Type I error of $\cP$}, defined as $\alpha(\cP) := \sup_{\pi \in \Delta(\Omega)}\PP_{Y \sim \pi, (X,R) \sim \cP}(Y \in R)$, is the maximum probability that an independent sample $Y$ is falsely rejected. The \emph{Type II error of $\cP$}, defined as $\beta(\cP) := \PP_{(X,R) \sim \cP}(X \notin R)$, is the probability that the sample $(X,R)$ from the joint probability $\cP$ is not detected.
\end{problem}

\ifdefined\isarxiv
\begin{wrapfigure}{r}{0.6\textwidth}
    \begin{center}
    \includegraphics[width=0.6\textwidth]{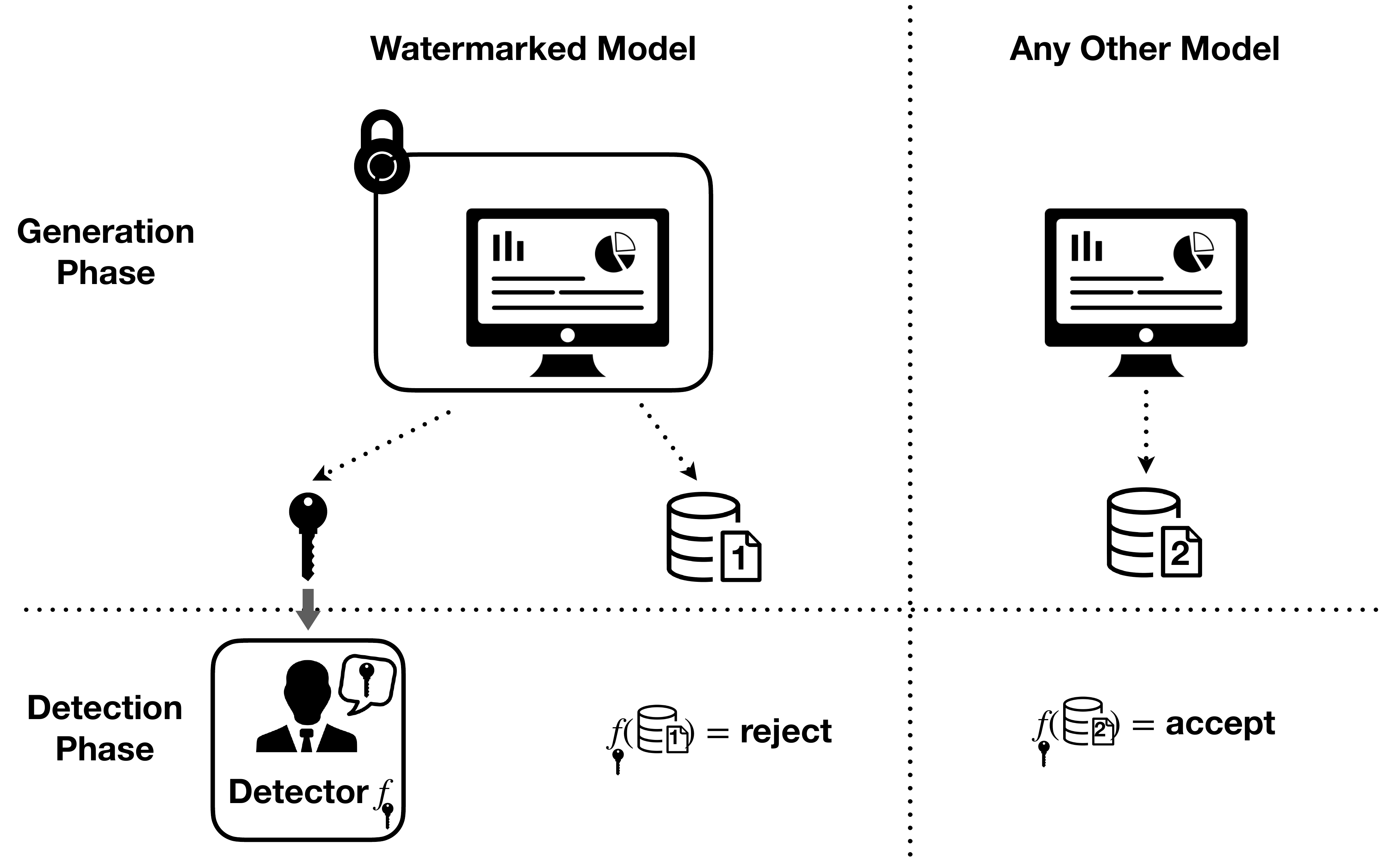}
  \end{center}
    \caption{Illustration of watermarking in practice.}
    \label{fig:watermark}
\end{wrapfigure}
\else
\begin{figure}
    \begin{center}
    \includegraphics[width=0.4\textwidth]{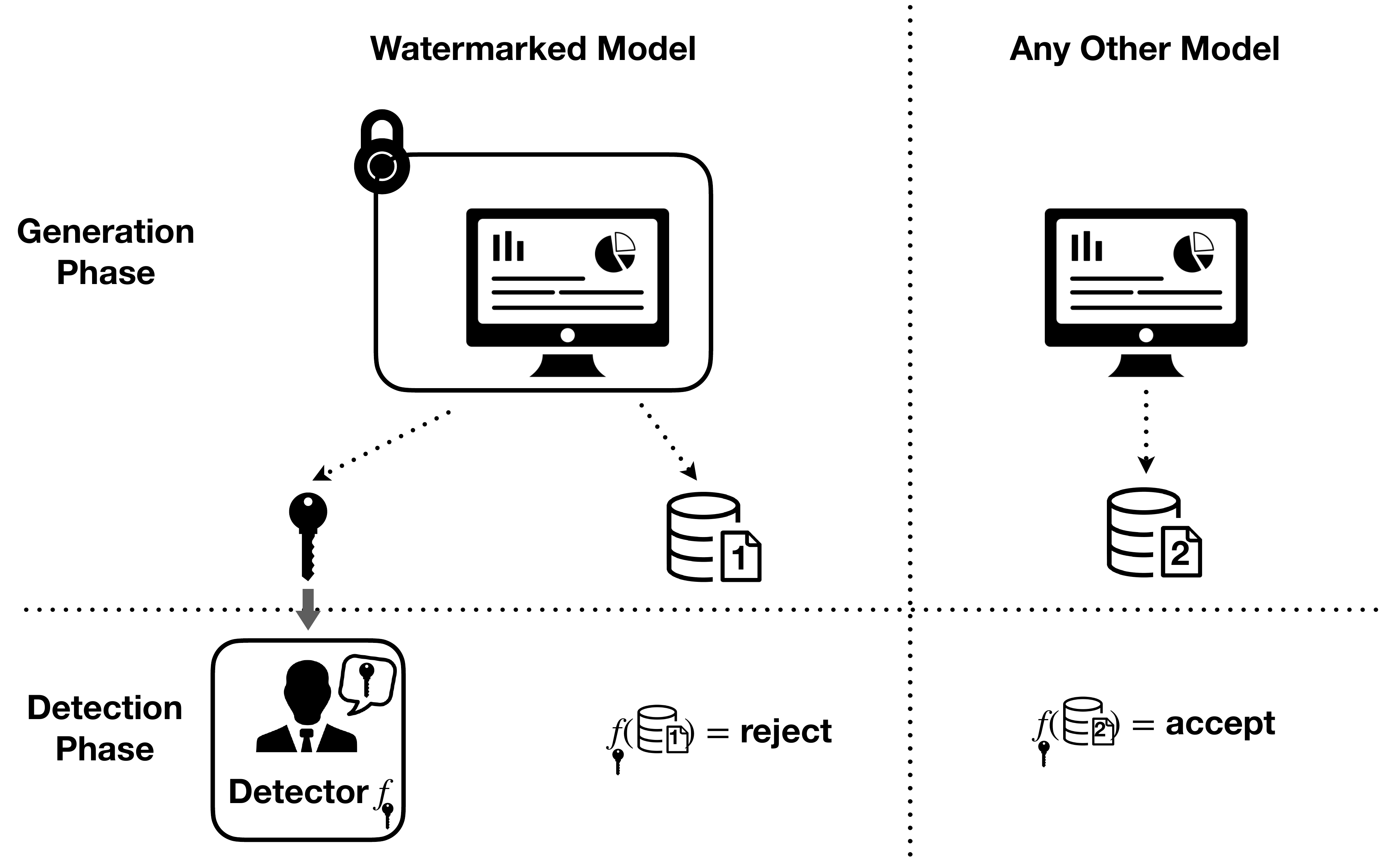}
  \end{center}
    \caption{Illustration of watermarking in practice.}
    \label{fig:watermark}
\end{figure}
% \vspace{-2em}
\fi

% We discuss examples of statistical watermarking in Appendix~\ref{sec:examples}. 
A few remarks are in order.

\begin{remark}[Difference between classical hypothesis testing]
In classical hypothesis testing, the rejection region is often nonrandomized or independent from the test statistics. However, in watermarking problem, the service provider has the incentive to facilitate the detection. The key insight is that $\cP$ is a coupling of the random output $X$ and the random rejection region $R$, so that $X \in R$ occurs with a high probability (low Type II error), while any independent sample $Y$ lies in $R$ with a low probability (low Type I error). 
\end{remark}

\begin{remark}[Implementation]\label{rmk:implementation}
In fact, it is imperative for the detector to observe the rejection region that is coupled with the output: otherwise, the output from the service provider and another independent output from the same marginal distribution would be statistically indistinguishable.

In practice, the process of coupling and sending the rejection region can be implemented by cryptographical techniques: the service provider could hash a secret key $\texttt{sk}$, and use pseudo-random functions $F_1,F_2$ to generate $(X,R) = (F_1(\texttt{sk}),F_2(\texttt{sk}))$. Now it suffices to send the secret key to the detector, who can then reproduce the reject region using the pseudo-random function $F_2$. This process is illustrated in Figure~\ref{fig:watermark}. 

By introducing the coupled and random rejection region, we abstract away the minutiae of cryptographical implementations, therefore allowing us to focus solely on the statistical trade-offs.
\end{remark}

For practical applications, it is additionally desirable for watermarking schemes to be model-agnostic, i.e, the marginal distribution of the rejection region is irrelevant to the watermarked distribution. Recall from Remark~\ref{rmk:implementation} that in practice, detectors usually adopt a pseudo-random function to generate the reject region from the shared secret keys. If the watermarking scheme $\cP$ depends on the underlying distribution $\rho$, then the pseudo-random function, and effectively the detector, need to know $\rho$. On the other hand, model-agnostic watermarking enables the detector to use a fixed, pre-determined pseudo-random function to generate the reject region, and hence perform hypothesis-testing \emph{without the knowledge of the underlying model that generates the output}. This is an important property enjoyed by existing watermarks~\citep{aaronson2022watermarking,kirchenbauer2023watermark,christ2023undetectable,kuditipudi2023robust}. Therefore in the following, we formulate model-agnostic within our hypothesis testing framework.

\begin{problem}[Model-Agnostic Watermarking]\label{prob:agnostic_watermark}
Given a sample space $\Omega$ and a set $\mathcal{Q} \subset  \Delta(\Omega)$, a $\mathcal{Q}$-watermarking scheme is a tuple $(\eta, \{\cP_\rho\}_{\rho \in \mathcal{Q}})$ where $\eta$ is a probability measure over $2^\Omega$, such that for any probability measure $\rho \in \mathcal{Q}$, $\cP_\rho$ is a distortion-free watermarking scheme of $\rho$ and its marginal distribution over $2^\Omega$, $\cP_\rho(\Omega,\cdot)$, equals $\eta(\cdot)$.

A model-agnostic watermarking scheme is a $\Delta(\Omega)$-watermarking scheme. 
\end{problem}

\begin{remark}[Information of the model]
\label{rmk:info-reveal}
A $\mathcal{Q}$-watermarking scheme can be interpreted as a way to watermark all distributions in the set $\mathcal{Q}$ while revealing no information of the model used to generate the output other than the membership inside $\mathcal{Q}$ (i.e., observing the rejection region, one is only able to infer that the output comes from a model in $\mathcal{Q}$, but is unable to know which exactly the model is). By letting $\mathcal{Q}$ be $\Delta(\Omega)$, model-agnostic watermarking thus reveals no information of the model.
\end{remark}

\subsection{Examples}
\label{sec:examples}

In the following examples, we show how existing watermarking schemes fit in our framework.

\begin{example}[Text Generation with Soft Red List, \citet{kirchenbauer2023watermark}]\label{ex:red-list}
In Algorithm 2 of \citet{kirchenbauer2023watermark}, the watermarking scheme (over sample space $\Omega = V^*$ where $V$ is the `vocabulary', i.e., the set of all tokens) of $\rho$ is given as follows:
\begin{itemize}%[itemsep=6pt]
    \item Fix threshold $C \in \R$, green list size $\gamma \in (0,1)$, and hardness parameter $\delta > 0$
    \item For $i = 1,2,\dots$
    \begin{itemize}%[itemsep=6pt]%[itemsep=4pt]
        \item Randomly partition $V$ into a green list $G$ of size $\gamma |V|$, and a red list $R$ of size $(1 - \gamma)|V|$.
    \item Sample the token $X_i$ from the following distribution from $\PP$ where $\PP(X_i = x) = $
    \begin{align*}
        \begin{cases}
            \frac{\rho(x) \cdot \exp(\delta)}{\sum_{x \in G}\rho(x) \cdot \exp(\delta) + \sum_{x \in R}\rho(x)}, &~ \text{if}~ x \in G\\
            \frac{\rho(x) }{\sum_{x \in G}\rho(x) \cdot \exp(\delta) + \sum_{x \in R}\rho(x)}, &~ \text{if}~ x \in R
        \end{cases}
    \end{align*}
    \end{itemize}
    \item Let the rejection region $R$ be
    \begin{align*}
        \left\{ X \in \Omega \  : \ \text{the number of green list tokens in X} \geq C\right\}.
    \end{align*}
    % \Hanlin{maybe more rigorously
    
    % $R = \left\{ X \in \Omega \  | \ \text{the number of green list tokens in X} \geq C\right\}.$}
\end{itemize}
The above sampling procedures as a whole define the joint distribution of the output $X = X_1X_2\cdots$ and the rejection region $R$, i.e., the $\Theta(\delta)$-distorted watermarking scheme $\cP_{\textsc{SoftRedList}}$. The detector observes the rejection region via the secret key that the service provider uses to generate the green and red lists.
\end{example}

\begin{example}[Complete watermarking algorithm $\mathrm{Wak}_{\mathrm{sk}}$, \citet{christ2023undetectable}]\label{ex:undetectable}
In Algorithm 3 of \citet{christ2023undetectable}, the watermarking scheme (over sample space $\Omega = \{0,1\}^*$) of $\rho$ is given as follows:
\begin{itemize}%[itemsep=6pt]%%[itemsep=6pt]
    \item Fix threshold $C \in \R$ and entropy threshold $\lambda > 0$
    \item Select $i$ such that the empirical entropy of $X_1X_2\dots X_i$ is greater than or equal to $\lambda$ %, green list size $\gamma \in (0,1)$, and hardness parameter $\delta > 0$
    \item For $j = i+1,i+2,\dots$
    \begin{itemize}%[itemsep=6pt]%[itemsep=4pt]
        \item Sample $u_j\in [0,1]$ uniformly at random.
    \item Let the binary token $X_j$ be given by $X_j = \begin{cases}
            1, &~ \text{if}~ u_j \leq \rho(1|X_1,\dots,X_{j-1})\\
            0, &~ \text{otherwise}
        \end{cases}$.
    % \begin{align*}
    %     X_j = \begin{cases}
    %         1, &~ \text{if}~ u_j \leq \rho(1|X_1,\dots,X_{j-1})\\
    %         0, &~ \text{otherwise}
    %     \end{cases}
    % \end{align*}
    \end{itemize}
    \item Let the rejection region $R$ be given by
    \begin{align*}
        \left\{ X: \sum_{j=i+1}^L \log \frac{1}{X_j u_j + (1-X_j)(1-u_j)} \geq C\right\}.
    \end{align*}
\end{itemize}
% The above sampling procedures as a whole define the watermarking scheme $\cP_{\mathrm{Wak}_{\mathrm{sk}}}$.
The above sampling procedures as a whole define the joint distribution of the output $X = X_1X_2\cdots$ and the rejection region $R$, i.e., the $0$-distorted watermarking scheme $\cP_{\mathrm{Wak}_{\mathrm{sk}}}$.The detector observes the rejection region via the index $i$ and $u_j (j > i)$.
\end{example}

\begin{example}[Inverse transform sampling $\mathrm{Wak}_{\mathrm{ITS}}$, \citet{kuditipudi2023robust}]\label{ex:inverse}
The inverse transform sampling scheme in \citet{christ2023undetectable} (over sample space $\Omega = [N]^*$) of $\rho$ is given as follows:
\begin{itemize}%[itemsep=6pt]%%[itemsep=6pt]
    \item Fix threshold $C \in \R$, resample size $T$, and block size $k$
    \item For $j = 1,2,\dots, $
    \begin{itemize}%[itemsep=6pt]%[itemsep=4pt]
        \item Let $\mu \leftarrow \rho(\cdot|X_1,\dots,X_{j-1})$.
        \item Sample $\xi_j = (u_j,\pi_j), \xi_j^{(t)} = (u_j',\pi_j')$ ($t = 1,\dots,T$) i.i.d. according to the following distribution:
        \begin{itemize}%[itemsep=6pt]
            \item Sample $u \in [0,1]$ uniformly at random;
            \item Sample $\pi$ uniformly at random from the space of permutations over the vocabulary $[N]$.
        \end{itemize}
    \item Let the token $X_j$ be given by 
    $$\pi^{-1}\left(\min\{\pi(i): \mu(\{j: \pi(j) \leq \pi(i) \}) \geq u\}\right).$$
    \end{itemize}
    \item Let the rejection region $R$ be
    \begin{align*}
        R = \left\{ X: \frac{1+\sum_{t=1}^{T} \mathbbm{1}\left(\phi(X,\xi^{(t)}) \leq \phi(X,\xi) \right)}{T+1}\leq C\right\}
    \end{align*}
    where $\xi = (\xi_1,\dots,\xi_{\mathrm{len}(X)})$, $\xi^{(t)} = (\xi_1^{(t)},\dots,\xi_{\mathrm{len}(X)}^{(t)})$, and $\phi(y,\xi)$ is given by
    \begin{align*}
        \underset{\substack{i = 1,\dots,\mathrm{len}(y)-k+1, \\j = 1,\dots,\mathrm{len}(\xi)}}{\min}\left\{d\left(\{y_{i+l}\}_{l=1}^{k-1},\{\xi_{(j+l)\% \mathrm{len}(\xi)}\}_{l=1}^{k-1}\right)\right\}
    \end{align*}
\end{itemize}
Here $d$ is an alignment cost, set as $d(y,(u,\pi)) = \sum_{i=1}^{\mathrm{len}(y)} \left|u_i - \frac{\pi_i(y_i)-1}{N-1}\right|$ 
% \begin{align*}
%     d(y,(u,\pi)) = \sum_{i=1}^{\mathrm{len}(y)} \left|u_i - \frac{\pi_i(y_i)-1}{N-1}\right|
% \end{align*}
in \citet{kuditipudi2023robust}. Additionally, a single permutation $\pi ~(\forall j,t)$ is used to reduce computation overhead. 
The above sampling procedures as a whole define the joint distribution of the output $X = X_1X_2\cdots$ and the rejection region $R$ in $\mathrm{Wak}_{\mathrm{ITS}}$.The detector observes the rejection region via $\xi,\xi'$.
\end{example}

Using similar approaches as in the above examples, we can encompass the methods of a number of works~\citep{aaronson2022watermarking,liu2023private,zhao2023provable,kuditipudi2023robust} into our framework.

\section{Statistical Limit in Watermarking}

\subsection{Rates under the general setting of Problem~\ref{prob:crypto_stat_watermark}}

Given the formulation of statistical watermarking, it is demanding to understand its statistical limit. In this section, we study the following notion of Uniformly Most Powerful (UMP) test, i.e., the watermarking scheme that achieves the minimum achievable Type II error among all possible tests with Type I error $\leq \alpha$. 

\begin{definition}[Uniformly Most Powerful Watermark]
A watermarking scheme $\cP$ is called \emph{Uniformly Most Powerful (UMP) $\epsilon$-distorted watermark of level $\alpha$}, if it achieves the minimum achievable Type II error among all $\epsilon$-distorted watermarking with Type I error $\leq \alpha$. 
\end{definition}

The following result gives an exact characterization of the UMP watermark and its Type II error.

\begin{theorem}\label{thm:crypto_stat_rates}
For probability measure $\rho$, the Uniformly Most Powerful $\epsilon$-distorted watermark of level $\alpha$, denoted by $\cP^*$, is given by
\begin{align*}
    \cP^*(X = x, R = R_0) = \begin{cases}
    \rho^*(x) \cdot \left(1 \wedge \frac{\alpha}{\rho^*(x)}\right), &~ R_0 = \{x\}\\
    \rho^*(x) \cdot \left(1- \frac{\alpha}{\rho^*(x)}\right)_+, &~ R_0 = \emptyset\\
    0, &~ \text{else}
\end{cases}
\end{align*}
where $\rho^* = \arg \min_{\TV(\rho'\|\rho) \leq \epsilon} \sum_{x \in \Omega: \rho'(x) > \alpha} \left(\rho'(x) - \alpha\right).$ 
Its Type II error is given by 
$$\min_{\TV(\rho'\|\rho) \leq \epsilon} \sum_{x \in \Omega: \rho'(x) > \alpha} \left(\rho'(x) - \alpha\right),$$
and when $|\Omega| \geq \frac{1}{\alpha}$ it simplifies to
\begin{align}\label{eq:optimal_type_2_err}
    \left(\sum_{x \in \Omega: \rho(x) > \alpha} \left(\rho(x) - \alpha\right) - \epsilon\right)_+.
\end{align}
% when $|\Omega| \geq \frac{1}{\alpha}$.
\end{theorem}

The key insight for proving \Cref{thm:crypto_stat_rates} is that maximizing Type II error over level $\alpha$ can be written as a linear program over the coupling distribution $\cP$. The detailed proof is deferred to Appendix~\ref{sec:crypto_stat_rates_proof}. In the following, we make a few remarks on Theorem~\ref{thm:crypto_stat_rates}.

\begin{remark}[Dependence on distortion parameter $\epsilon$]
As seen from the theorem, when a larger distortion parameter $\epsilon$ is allowed, the Type II error would decrease. This aligns with the intuition that adding statistical bias would make the output easier to detect~\citep{aaronson2022my,kirchenbauer2023watermark}. Among all choices of $\epsilon$, the case $\epsilon=0$ is of particular interest since it preserves the marginal distribution of the service provider's output. Therefore, we will focus on this distortion-free case in the following sections.
\end{remark}

\begin{remark}[Intuition behind $\cP^*$]
Recall that in practice, the watermarks are implemented via pseudo-random functions. Therefore, the uniformly most powerful test in \Cref{thm:crypto_stat_rates} is effectively using a pseudo-random generator to approximate the distribution $\rho$, combined with an $\alpha$-clipping to control Type I error. 
This construction reveals a surprising message: simply using pseudo-random generator to approximate the distribution is optimal.
\end{remark}

\begin{remark}[Watermarking guarantees]
To achieve the upper bound of \Cref{thm:crypto_stat_rates}, the detector needs to access the model and the prompt in order to generate the reject region, which is not always accessible in many real-world applications. Therefore, the upper bound of \Cref{thm:crypto_stat_rates} achieves a weaker watermarking guarantee compared with previous works~\citep{aaronson2022my,kirchenbauer2023watermark,christ2023undetectable}. In Section~\ref{sec:agnostic}, we study model-agnostic watermarking that overcomes this limitation.

Nonetheless, the lower bound in \Cref{thm:crypto_stat_rates} characterizes a fundamental limit of Problem~\ref{prob:crypto_stat_watermark}, thus providing an information-theoretic lower bound for all watermarks.
\end{remark}

\begin{remark}[Implementation]
To implement the UMP watermark using a predetermined key, one may apply the key to the random seeds used in model generation, and sets the reject region to be the output with probability $1 \wedge \frac{\alpha}{\rho^*(x)}$. To implement the UMP watermark without the detector's knowledge on the secret key, one could hash the first few tokens to seed the pseudo-random function. In summary, the UMP watermark could use the same key to watermark many outputs, and the key needs not to be generated at the same time as the output itself.
\end{remark}

\begin{remark}[Use cases of the UMP watermark]
The utilization of the UMP watermark offers an efficient approach for (language model) service providers to determine if instruction-following datasets have been generated by a specific model. In the context of instruction-following datasets, both the prompt and response are explicitly provided to the detectors, enabling the UMP watermark to perform accurate watermarking and detection without extra source of information. This usage is beneficial in identifying and filtering out data points that have been comtaminated by texts generated from models like GPT-4~\citep{openai2023gpt}, thereby preserving the purity and quality of the training data.
\end{remark}

\begin{remark}[Dependence on the randomness of $\rho$]\label{rmk:randomness}
If $\rho$ is deterministic, the Type II error
$\left(\sum_{x \in \Omega: \rho(x) > \alpha} \left(\rho(x) - \alpha\right) - \epsilon\right)_+$
reduces to $1 - \alpha - \epsilon$ and shows limited practical utility of statistical watermarking. This is expected since when the service provider deterministically outputs $z$, it would be impossible to distinguish the watermark distribution with an independent output from $\delta_z$. In general, \Cref{thm:crypto_stat_rates} implies that the Type II error decreases when the randomness in $\rho$ increases, matching the reasoning in previous works~\cite{aaronson2022my,christ2023undetectable}.
\end{remark}

% The proof of \Cref{thm:crypto_stat_rates} is deferred to Appendix~\ref{sec:crypto_stat_rates_proof}

\subsection{Rates of model-agnostic watermarking}\label{sec:agnostic}

It is noticeable that for large $\mathcal{Q}$, a $\mathcal{Q}$-watermarking scheme can not perform as good as a watermarking specifically designed for $\rho$ for any distribution $\rho \in \mathcal{Q}$. This means that Uniformly Most Powerful $\mathcal{Q}$-Watermarking might not exist in general. To evaluate model-agnostic watermarking schemes, a natural desideratum is therefore the maximum difference between its Type II error and the Type II error of the UMP watermarking of $\rho$ over all distributions $\rho$, under fixed Type I error. Specifically, we introduce the following notion. 

\begin{definition}[Minimax most powerful model-agnostic watermark]
We say that a $\mathcal{Q}$-agnostic watermark $(\eta, \{\cP_\rho\}_{\rho \in \mathcal{Q}})$ is of level-$\alpha$ if the Type I error of $\cP_\rho$ is less than or equal to $\alpha$ for any $\rho \in \mathcal{Q}$. Define the \emph{maximum Type II error loss} of $(\eta, \{\cP_\rho\}_{\rho \in \mathcal{Q}})$ as 
\begin{align*}
    \gamma(\eta) := \sup_{\rho \in \mathcal{Q}} \beta(\cP_\rho) - \beta(\cP_\rho^\ast)
\end{align*}
where $\cP_\rho^\ast$ is the UMP distortion-free watermark of $\rho$. 

We say that a $\mathcal{Q}$-agnostic watermarking scheme is minimax most powerful, if it minimizes the maximum Type II error loss among all $\mathcal{Q}$-agnostic watermarks of level $\alpha$.
\end{definition}

The following result characterizes the Type II error loss of the minimax most powerful model-agnostic watermarking.

\begin{theorem}\label{thm:agnostic}
Let $|\Omega| = n$ and suppose $\alpha n, \frac{1}{\alpha} \in \mathbb{Z}$\footnote{For the general case, it suffices to let $a_1 = 1/(\lceil 1/\alpha \rceil)$ and $n_1 = \lceil\alpha_1 n\rceil/\alpha_1$ and augment $\Omega$ with $n_1 - n$ dummy outcomes. Then $\alpha_1 n, 1/\alpha_1 \in \mathbb{Z}$ and hence the minimax bound for the new sample space with cardinality $n_1$ and the new Type I error $\alpha_1$ yields a nearly-matching bound for $(n,\alpha)$.}. In the minimax most powerful model-agnostic watermarking scheme of level-$\alpha$, the marginal distribution of the reject region is given by
\begin{align*}
    \eta^*(A) = \begin{cases}
        \frac{1}{{n \choose \alpha n}}, &~ \text{if}~ |A| = \alpha n\\
        0, &~ \text{otherwise}
    \end{cases}.
\end{align*}
The maximum Type II error loss of the minimax most powerful model-agnostic watermarking scheme of level-$\alpha$ is given by $\gamma(\eta^\ast) = \frac{{n - \frac{1}{\alpha} \choose \alpha n}}{{n \choose \alpha n}}$. 
In the regime $\alpha \to 0_+, n \to +\infty$, we have $\gamma(\eta^\ast) \to c$ for some constant $c \leq e^{-1}$, and when $1/(\alpha n) \to 0_+$ is further satisfied, $c = e^{-1}$.
\end{theorem}

The theorem establishes existence of $\cP_\rho$ for any $\rho$ without explicit construction. To grasp this concept, consider three sets: $U$, the output space; $V$, the set of reject regions; and $W$, the subset of $U \times V$ defined by $\{(u,v): u \in v\}$. Notice that the type II error is essentially $1- \cP_\rho(W)$. 
Therefore, our objective is to establish existence of a probability measure $P$ over $U \times V$ such that its marginal distributions align with $\eta$ and $\rho$, respectively, and the probability assigned to $W$, denoted as $P(W)$, meets a certain lower bound. This is the question studied by Strassen's theorem~\citep{strassen1965existence}, which stipulates conditions for the existence of such a measure. Hence, by verifying Strassen's conditions, we confirm the existence of the required measure without the necessity of explicitly constructing the coupling. We defer the detailed proof to Appendix~\ref{sec:agnostic_proof}.

\begin{remark}
Theorem~\ref{thm:agnostic} implies that for any distribution $\rho$, the Type II error of model-agnostic watermark is upper bounded by $\frac{{n - \frac{1}{\alpha} \choose \alpha n}}{{n \choose \alpha n}} + \sum_{x: \rho(x) \geq \alpha} (\rho(x) - \alpha).$ 
The convergence $\gamma(\eta^\ast) \to e^{-1}$ implies that the minimax optimal model-agnostic watermark exhibits an increase in Type II error by an additive factor of $e^{-1}$ compared to the UMP watermark in the worst-case scenario. 
\end{remark}

\begin{remark}\label{rmk:rate_contradiction}
The $e^{-1}$ maximum Type II error loss does not contradict with the $h^{-2}$ rates in previous works~\citep{aaronson2022my,christ2023undetectable,kuditipudi2023robust}, because as $n \gtrsim h^{-2}$, the model distribution (of the sequences of $n$ tokens with average entropy $h$ per token) is beyond the worst case. Indeed, such distributions have higher differential entropy than the hard instances in the proof. 
\end{remark}

Remark~\ref{rmk:rate_contradiction} highlights that the hard instance constructed in \Cref{thm:agnostic} may possess a lower entropy than that of the actual model. Therefore, it raises an important question: for a smaller class $\mathcal{Q}$ that contains distributions with higher entropy, what is the minimum achievable Type II error loss for $\mathcal{Q}$-agnostic watermarking? It is obvious that the minimax rate over a higher entropy level should improve upon the previous rate of $e^{-1}$.

Towards answering this question, we consider the following class of distributions: 
\begin{align*}
    \mathcal{Q}_\kappa:= \left\{\rho: \sup_{\omega \in \Omega}\rho(\{\omega\}) \leq \kappa\right\}
\end{align*}
where $\kappa$ represents the level of randomness and decreases as entropy increases. The \emph{maximum Type II error loss} of $\mathcal{Q}_\kappa$-agnostic watermarking $(\eta, \{\cP_\rho\}_{\rho \in \mathcal{Q}_\kappa})$ is thus given by 
\begin{align*}
    \gamma(\eta, \kappa) := \max_{\rho \in \Delta(\Omega): \sup_{\omega \in \Omega}\rho(\{\omega\}) \leq \kappa} \beta(\cP_\rho) - \beta(\cP_\rho^\ast)
\end{align*}
where $\cP_\rho^\ast$ is the UMP distortion-free watermark of $\rho$. The following result gives an upper bound of the above quantity, thus answering the question.

\begin{theorem}\label{thm:agnostic_level}
Let $|\Omega| = n$ and suppose $\alpha n, \frac{1}{\kappa} \in \mathbb{Z}$. 
Then the maximum Type II error loss of the minimax $\mathcal{Q}_\kappa$-agnostic watermarking of level-$\alpha$ is upper bounded by 
% $\gamma(\eta^\ast,\kappa) \leq \frac{{n - \alpha n \choose 1/\kappa}}{{n \choose 1/\kappa}}$.
$$\gamma(\eta^\ast,\kappa) \leq \frac{{n - \alpha n \choose 1/\kappa}}{{n \choose 1/\kappa}}.$$
\end{theorem}

The proof can be found in Appendix~\ref{sec:agnostic_level_proof}. When $\kappa \leq \alpha$, the bound $\frac{{n - \alpha n \choose 1/\kappa}}{{n \choose 1/\kappa}}$ improves over $e^{-1}$. In the next section, we will apply \Cref{thm:agnostic_level} to the i.i.d. setting where $\kappa$ can be exponentially small. This will lead to an negligible maximum Type II error loss for model-agnostic watermarking.

\subsection{Rates in the i.i.d. setting}\label{sec:iid_rates}

In practice, the sample space $\Omega$ is usually a Cartesian product in the form of $\Omega_0^{\otimes n}$. For example, in large language models, the output takes form of a sequence of tokens, each coming from the same vocabulary set $V$. The quantity of practical interest becomes the minimum number of tokens to achieve certain statistical watermarking guarantee. 
This demands specializing and transferring the results from  \Cref{thm:crypto_stat_rates} and \Cref{thm:agnostic_level} to deal with distributions in product measureable spaces, and finding the explicit rates of the minimum number of required tokens.

In this section, we consider the product distribution $\rho = \rho_0^{\otimes n}$ over $ \Omega_0^{\otimes n}$ and the important setting of $\epsilon = 0$ (distortion-free watermarking). We introduce the following two quantities:
\begin{itemize}%[itemsep=1pt]
    \item Let $h$ denote the entropy of $\rho_0$. We use $n_{\ump}(h,\alpha,\beta)$ to denote the minimum number of tokens required by the UMP watermark to achieve Type I error $\leq \alpha$ and Type II error $\leq \beta$.
\item Define $n_{\minmax}(h,\alpha,\beta)$ as the number of tokens required by minimax $\mathcal{Q}^h$-agnostic watermark to achieve Type I error $\leq \alpha$ and Type II error $\leq \beta$, where $\mathcal{Q}^h := \left\{\rho = \rho_0^{\otimes n}: H(\rho_0) \geq h\right\}$, i.e. 
contains all distributions $\rho = \rho_0^{\otimes n}$ such that the entropy of $\rho_0$ is $\geq h$.
\end{itemize}
Together, $n_{\ump}(h,\alpha,\beta)$ and $n_{\minmax}(h,\alpha,\beta)$ serve as critical thresholds beyond which the desired statistical conclusions can be drawn regarding the output, making them essential parameters in watermarking applications. 

We start by inspecting the rates in Theorem~\ref{thm:crypto_stat_rates} in the i.i.d. setting. The following result gives a nearly-matching upper bound and lower bound of $n_{\ump}(h,\alpha,\beta)$. 

\begin{theorem}\label{thm:iid_rates}
Suppose $\alpha, \beta < 0.1$. We have
\begin{align*}
    % n_{\ump}(h,\alpha,\beta) \geq \left(\frac{\ln \frac{\ln 2}{h} \cdot \left(\ln \frac{1}{2\alpha} \wedge \ln \frac{1}{2\beta}\right)}{2h}\right) \vee \frac{\ln\frac{1}{2\alpha}}{4h}.
    n_{\ump}(h,\alpha,\beta) \geq O\left(\left(\frac{\ln \frac{1}{h} \left(\ln \frac{1}{\alpha} \wedge \ln \frac{1}{\beta}\right)}{h}\right) \vee \frac{\ln\frac{1}{\alpha}}{h}\right).
\end{align*}
Furthermore, let $k = |\Omega_0|$, we have
\begin{align*}
    % n_{\ump}(h,\alpha,\beta) \leq &~ 
    % \left(\frac{400\ln \frac{9k}{h}\cdot \left(\ln \frac{1}{\alpha} \wedge \ln \frac{1}{\beta}\right)}{h}\right)\\ &~ \vee \frac{(4000+800\ln (9k))\ln\frac{1}{\alpha}}{h}.
    &~ n_{\ump}(h,\alpha,\beta)\\
    \leq &~
    \Omega\left(\left(\frac{\ln \frac{k}{h}\cdot \left(\ln \frac{1}{\alpha} \wedge \ln \frac{1}{\beta}\right)}{h}\right)\vee \frac{\ln\frac{1}{\alpha}\ln k}{h}\right).
\end{align*}
\end{theorem}

\begin{remark}[Tightness]
Up to a constant and logarithmic factor in $k$, our upper bound matches the lower bound. Notice that since any model with an arbitrary token set can be reduced into a model with a binary token set~\citep{christ2023undetectable} (i.e. $k=2$), our bound is therefore tight up to a constant factor.
\end{remark}

Using \Cref{thm:agnostic_level} and \Cref{thm:iid_rates}, we are now in the position to characterize $n_{\minmax}(h,\alpha,\beta)$. 
Suppose the sample space is a Cartesian product $\Omega = \Omega_0^{\otimes n_0}$ and constrain to product measures over sequences of $n_0$ tokens, like in Section~\ref{sec:iid_rates}. 
We start by the following relationship:\footnote{In the rest of this section, we omit logarithmic factors in the cardinality of the vocabulary.}
\begin{align*}
    1- \max_{\rho_0:H(\rho_0) \geq h}\max_{\omega \in \Omega_0} \rho_0(\{\omega\}) \geq {\Omega}\left(\frac{h}{\ln(1/h)}\right)
\end{align*}
where a detailed derivation can be found in Lemma~\ref{lem:tight_entropy_bound}. 
It follows that 
\begin{align*}
    \kappa \leq \left(\max_{\rho_0:H(\rho_0) \geq h}\max_{\omega \in \Omega_0} \rho_0(\{\omega\})\right)^{n_0} = e^{-{\Omega}\left(\frac{n_0 h}{\ln(1/h)}\right)}.
\end{align*}
Using this observation and the derivation in \Cref{thm:agnostic}, $\gamma(\eta^*,\kappa)$ can be bounded by 
\begin{align*}
    (1-\alpha)^{1/\kappa} \leq (1-\alpha)^{e^{{\Omega}\left(\frac{n_0h}{\ln(1/h)}\right)}}.
\end{align*}
This means that when $n_0 \gtrsim \frac{\ln(1/h)}{h}\cdot (\ln(1/\alpha) + \ln(1/\beta))$, the maximum Type II error loss given by \Cref{thm:agnostic_level} and the Type II error of the UMP watermarking given in Theorem~\ref{thm:iid_rates} can be simultaneously bounded by $\beta$, thus establishing an upper bound. Furthermore, this rate matches the lower bound in \Cref{thm:iid_rates}, where the guarantee is weaker (model-nonagnostic). Combining the above arguments, the following result is thus immediate.

\begin{corollary}\label{cor:iid_rates_agnostic}
Suppose $\alpha, \beta < 0.1$.  We have
\begin{align*}
    n_{\minmax}(h,\alpha,\beta) = {\Theta}\left(\frac{\ln(1/h)}{h}\cdot (\ln(1/\alpha) + \ln(1/\beta))\right).
\end{align*}
\end{corollary}

\begin{remark}[Comparison with previous works]
As commented in Remark~\ref{rmk:randomness}, the regime $h \ll 1$ is more important and challenging because it is the scenario where watermarking is difficult. 
In this regime, our rate of $\frac{\ln(1/h)}{h}$ improves the previous rate of $h^{-2}$ in a line of works~\citep{aaronson2022my,kirchenbauer2023watermark,zhao2023provable,liu2023private,kuditipudi2023robust}, and highlights a fundamental gap between the existing watermarks and the information-theoretic lower bound. 
\end{remark}
% The proof is deferred to Appendix~\ref{sec:iid_rates_proof}.

\section{Robust Watermarking}\label{sec:robust}
\ifdefined\isarxiv
\else
\begin{table*}[t]
\centering
\begin{tabular}{lcccc}
\toprule
\textbf{Scheme / Temperature} & \textbf{0} & \textbf{0.3} & \textbf{0.7} & \textbf{1} \\
\midrule
Distribution Shift~\citep{kirchenbauer2023watermark}        &   \textbf{65}  &  63           &    77          &    136                    \\
Exponential~\citep{aaronson2022watermarking}                   &     impossible      &     890         &       190       &      93      \\
Inverse Transform~\citep{kuditipudi2023robust}             &      impossible      &    $+\infty$        &  434            & 222            \\
Binary~\citep{christ2023undetectable}                        &         impossible   &   $+\infty$           & $+\infty$             &      386      \\
Ours                          &     impossible     &  \textbf{60.5}              &      \textbf{24}        &       \textbf{15}     \\
\bottomrule
\end{tabular}
\caption{Comparison of our watermark scheme (in the model non-agnostic setting) to previous works tested on the \textsc{MarkMyWords} benchmark by \citep{piet2023mark}. For each watermark scheme and each temperature, we show the (average) minimum number of tokens needed to detect the watermark under the constraint that type I error is less than $\alpha = 0.02$. For the first four rows, one can refer to Figure 1 of \citet{piet2023mark}; $+\infty$ means over half of all generations are not watermarked and ``impossible'' means when the temperature is 0, the text generation procedure is deterministic and the entropy is zero, and thus any distortion-free watermark scheme does not work.}
\label{exp::table}
\end{table*}
% \vspace{-1em}
\fi
In the context of watermarking large language models, it's crucial to acknowledge users' capability to modify or manipulate model outputs. These modifications include cropping, paraphrasing, and translating the text, all of which may be employed to subvert watermark detection. Therefore, in this section, we introduce a graphical framework, modified from Problem~\ref{prob:crypto_stat_watermark}, to account for potential user perturbations and investigate the optimal watermarking schemes robust to these perturbations. The formulation here shares similarity with a concurrent work by~\citet{zhang2023watermarks}.

\begin{definition}[Perturbation graph]\label{def:perturb_graph}
A perturbation graph over the discrete sample space $\Omega$ is a directed graph $G = (V,E)$ where $V$ equals $\Omega$ and $(u,u) \in E$ for any $u \in V$. For any $v \in V$, let $in(v) = \{w\in V: (w,v) \in E\}$ denote the set of vertices with incoming edges to $v$, and let $out(v) = \{w\in V: (v,w) \in E\}$ denote the set of vertices with outcoming edges from $v$. 
\end{definition}

The perturbation graph specifies all the possible perturbations that could be made by the user: any $u \in V$ can be perturbed into $v \in V$ if and only if $(u,v)\in E$, i.e., there exists a directed edge from $u$ to $v$.

\begin{example}
Consider $\Omega = \Omega_0^{\otimes n}$. Let the user have the capacity to change no more than $c$ tokens, i.e., perturb any sequence of tokens $x = x_1x_2\cdots x_n$ to another sequence $y = y_1 y_2\cdots y_n$ with Hamming distance less than or equal to $c$. Then the perturbation graph is given by $G = (V, E)$ where $V = \Omega^n$ and $E = \{(u,v): u,v \in V, d(u,v) \leq c\}$ ($d$ is the Hamming distance, i.e., $d(x,y) = \sum_{i=1}^n \mathbbm{1}(x_i \neq y_i)$).
\end{example}

\begin{problem}[Robust watermarking scheme]\label{prob:robust_crypto_stat_watermark}
A robust watermarking scheme with respect to a perturbation graph $G$ is a watermarking scheme except that its Type II error is defined as $\E_{X,R \sim \cP}\left[\max_{Y \in out(X)}\mathbbm{1}(Y \notin R)\right]$, i.e., the probability of false negative given that the user adversarially perturbs the output.
\end{problem}

The next result characterize the optimum Type II error achievable by robust watermarking, where the proof can be found in Appendix~\ref{sec:robust_rates_proof}.

\begin{theorem}\label{thm:robust_rates}
Define the shrinkage operator $\cS_G: 2^\Omega \to 2^{\Omega}$ (of a perturbation graph $G$) by $\cS_G(R) = \{x \in \Omega: out(x) \subset R\}$ and its inverse $\cS_G^{-1}(R) = \cup_{x \in R}out(x) $. Then the minimum Type II error of the robust, $0$-distorted UMP test of level $\alpha$ in Problem~\ref{prob:robust_crypto_stat_watermark} is given by the solution of the following Linear Program
\begin{align}\label{eq:robust_lp}
    \min_{x \in \R^{|\Omega|}} &~ 1- \sum_{y \in \Omega} \rho(y) x(y)\\
    s.t. &~ \sum_{y \in in(z)} \rho(y) x(y) \leq \alpha, \sum_{z \in \Omega}x(z) \leq 1,\notag\\
    &~ 0 \leq x(z) \leq 1, ~\forall z \in \Omega.\notag
\end{align}
The UMP watermarking is given by $\cP^*(X = y, R = R_0)$
\begin{align*}
    = \begin{cases}
    \rho(y) \cdot x^*(y), &~ R_0 = \cS_G^{-1}(\{y\})\\
    \rho(y) \cdot \left(1- x^*(y)\right), &~ R_0 = \emptyset\\
    0, &~ \text{otherwise}
\end{cases}
\end{align*}
where $x^*$ is the solution of Eq.~\eqref{eq:robust_lp}.
\end{theorem}

\begin{remark}[Dependence on the sparsity of graph]
From Eq.~\eqref{eq:robust_lp}, we observe that the perturbation graph influence the optimal Type II error via the constraint set. Indeed, if the graph is dense, the constraints $\sum_{y \in in(z)} \rho(y) x(y) \leq \alpha$ involve many entries of $y \in \Omega$ and thus decrease the value $\sum_{y \in \Omega} \rho(y) x(y)$, thereby increasing the Type II error. On the other extreme, when the edge set of the perturbation graph is $E = \{(u,u): u \in v\}$, i.e., the user can not perturb the output to a different value, then optimum of Eq.~\eqref{eq:robust_lp} reduces to Eq.~\eqref{eq:optimal_type_2_err} (setting $\epsilon = 0$).
\end{remark}
% The proof is deferred to Appendix~\ref{sec:robust_rates_proof}.

\section{Experiments}
\ifdefined\isarxiv
\begin{table*}[t]
\centering
\begin{tabular}{lcccc}
\toprule
\textbf{Scheme / Temperature} & \textbf{0} & \textbf{0.3} & \textbf{0.7} & \textbf{1} \\
\midrule
Distribution Shift~\citep{kirchenbauer2023watermark}        &   \textbf{65}  &  63           &    77          &    136                    \\
Exponential~\citep{aaronson2022watermarking}                   &     impossible      &     890         &       190       &      93      \\
Inverse Transform~\citep{kuditipudi2023robust}             &      impossible      &    $+\infty$        &  434            & 222            \\
Binary~\citep{christ2023undetectable}                        &         impossible   &   $+\infty$           & $+\infty$             &      386      \\
Ours                          &     impossible     &  \textbf{60.5}              &      \textbf{24}        &       \textbf{15}     \\
\bottomrule
\end{tabular}
\caption{Comparison of our watermark scheme (in the model non-agnostic setting) to previous works tested on the \textsc{MarkMyWords} benchmark by \citep{piet2023mark}. For each watermark scheme and each temperature, we show the (average) minimum number of tokens needed to detect the watermark under the constraint that type I error is less than $\alpha = 0.02$. For the first four rows, one can refer to Figure 1 of \citet{piet2023mark}; $+\infty$ means over half of all generations are not watermarked and ``impossible'' means when the temperature is 0, the text generation procedure is deterministic and the entropy is zero, and thus any distortion-free watermark scheme does not work.}
\label{exp::table}
\end{table*}
\else
\fi

In this section, we show experimental results comparing our watermark scheme to several previous works. We test our watermark scheme on the \textsc{MarkMyWords} benchmark by \citet{piet2023mark}. \Cref{exp::table} shows the average number of tokens needed to detect the watermark for five different watermark schemes under different temperatures on the \textsc{MarkMyWords} benchmark. We choose Llama2-7B-chat~\citep{touvron2023llama} as the model to be watermarked and enforce that the type I error is less than $\alpha = 0.02$.

\Cref{exp::table} shows that our watermark scheme needs significantly fewer tokens to detect the watermark in the model non-agnostic setting, which provides strong empirical evidence that our watermark scheme is statistically optimal (\Cref{thm:iid_rates}). An exception is that for the distribution shift scheme~\citep{kirchenbauer2023reliability} with low temperature 0.3, the number of tokens required is only slightly larger than our scheme because the distribution shift scheme is not distortion-free. Note that the comparison in \Cref{exp::table} is made under the model non-agnostic setting (the rate in the model-nonagnostic setting is not fundamentally different from that in the model-agnostic setting, due to Corollary~\ref{cor:iid_rates_agnostic}) without considering robustness, while the four previous schemes also work for model agnostic setting with robustness guarantees. 
Therefore, our experiments corroborate the improved statistical trade-offs and highlight the fundamental gap, instead of advocating for the superiority of any particular watermarking scheme.

\section{Conclusions}

The understanding of watermarking large language models is advanced by framing it within the paradigm of hypothesis testing. 
We find that using a pseudo-random generator to approximate the model distribution (with probability clipping) yields the optimal Type II error among all level-$\alpha$ tests. Model-agnostic watermarking, reflecting the practical scenarios where the detector does not have access to the model distribution, enjoys a minimax bound in Type II errors depending on the model class. 
In the context where the output is a sequence of several tokens, we find that the optimal number of i.i.d. tokens required to detect statistical watermarks is $h^{-1}\log (1/h)$, improving upon the previous rate of $h^{-2}$ and highlighting a fundamental gap. Finally, the optimal Type II error of robust UMP watermarking can be characterized via a linear program, which exhibits the trade-off between robustness and detectability.

Watermarking is an essential technique to diminish the misuse of large language models. It tackles several critical social issues concerning the malicious usage of language models such as the contamination of datasets, academic misconduct, creation of fake news, and circulation of misinformation. By laying the theoretical foundation of statistical watermarking, our paper provides unifying and systematic approach to evaluate the statistical guarantees of existing and future watermarking schemes, elucidating the statistical limit of (robust) watermarking problems, and revealing the optimal rates in the important setting of i.i.d. tokens. In the above ways, our work contributes to the research endeavours on addressing these societal issues in language modelling, thus having potentially positive social impacts.

\section*{Acknowledgements}
We would like to thank Julien Piet for his invaluable assistance with the experiments. Additionally, we are thankful to Or Zamir for his insightful comments on the earlier version of this manuscript.

\clearpage
\ifdefined\isarxiv
\else
\clearpage
\section*{Statement of Potential Broader Impacts}

Watermarking is an essential technique to diminish the misuse of large language models. It tackles several critical social issues concerning the malicious usage of language models such as the contamination of datasets, academic misconduct of students, creation of fake news, and circulation of misinformation. By laying the theoretical foundation of statistical watermarking, our paper provides unifying and systematic approach to evaluate the statistical guarantees of existing and future watermarking schemes, elucidating the statistical limit of (robust) watermarking problems, and revealing the optimal rates in the important setting of i.i.d. tokens. In the above ways, our work contributes to the research endeavours on addressing these societal issues in language modelling. In conclusion, our paper has positive social impacts.
\fi

\ifdefined\isarxiv
\bibliography{ref}
\bibliographystyle{iclr2024_conference}

\newpage
\appendix
\else

\bibliography{ref}
\bibliographystyle{icml2024}

%%%%%%%%%%%%%%%%%%%%%%%%%%%%%%%%%%%%%%%%%%%%%%%%%%%%%%%%%%%%%%%%%%%%%%%%%%%%%%%
%%%%%%%%%%%%%%%%%%%%%%%%%%%%%%%%%%%%%%%%%%%%%%%%%%%%%%%%%%%%%%%%%%%%%%%%%%%%%%%
% APPENDIX
%%%%%%%%%%%%%%%%%%%%%%%%%%%%%%%%%%%%%%%%%%%%%%%%%%%%%%%%%%%%%%%%%%%%%%%%%%%%%%%
%%%%%%%%%%%%%%%%%%%%%%%%%%%%%%%%%%%%%%%%%%%%%%%%%%%%%%%%%%%%%%%%%%%%%%%%%%%%%%%
\newpage
\appendix
\onecolumn

\fi

\ifdefined\isarxiv
\else
\section{Additional Related Works}
\label{sec:add_related_works}
% Watermarking is a powerful white-box method for detecting LLM-generated texts~\citep{tang2023science}. 
% Watermarks can be injected either into a pre-existing text (edit-based watermarks) or during the text generation (generative watermarks), while our work falls in the latter category. Edit-based watermarking~\citep{rizzo2019fine,abdelnabi2021adversarial,yang2022tracing,kamaruddin2018review} has been the focus of several studies in the past. The concept of generative watermarking dates back to the work of~\citet{
% venugopal-etal-2011-watermarking}, while our work is more relevant to a recent line of works~\citep{aaronson2022my,kirchenbauer2023watermark,kuditipudi2023robust,christ2023undetectable} that introduce statistical signals into text generation. Specifically, \citet{kirchenbauer2023watermark} increases the probability that tokens are chosen from a randomly sampled `green' list; \citet{aaronson2022my} selects the token $i$ that maximizes keys randomly sampled from exponential distributions with mean $1/p_i$; \citet{christ2023undetectable} samples the tokens by solving the optimal transport from uniform distribution in $[0,1]$; \citet{kuditipudi2023robust} introduces inverse transform sampling as a distortion-free watermarking method. These watermarks are evaluated in the benchmark of~\citet{piet2023mark}.

Statistical watermarking techniques share the similarity that the outputs are correlated with some secret keys (which could come from either external randomness or internal hashing), thereby coupling the rejection region and the outputs in the hypothesis testing. This fact is recognized by recent works of~\citet{kuditipudi2023robust,zhao2023provable}, where model-agnosticity in the detection phase is also emphasized. The exponential scheme in~\citet{aaronson2022watermarking}, inverse transform sampling scheme in~\citet{kuditipudi2023robust}, and binary scheme in~\citet{christ2023undetectable} come with theoretical guarantees that (i) the watermarked model distribution cannot be distinguished from the original distribution (called undetectability~\citep{christ2023undetectable} or distortion-freeness~\citep{kuditipudi2023robust}), and (ii) the outputs from watermarked models are statistically detectable as long as the entropy is lower bounded. In contrast, \citet{kirchenbauer2023watermark} is not distortion-free, nonetheless enjoying little degradation in generation quality and provable detectability~\citep{zhao2023provable} with suitable parameter choice of the bias parameter (logits increase $\delta$ in the `green' list). Despite the aforementioned theoretical efforts in establishing guarantees for existing watermarks, the fundamental tradeoff in this hypothesis testing problem and the rates on the required number of generated tokens remain unsolved.  

% Our framework of formulating the watermark problem as a hypothesis testing is general and subsumes previous frameworks such as \citet{kirchenbauer2023watermark,christ2023undetectable}. A recent work \citet{kuditipudi2023robust} similarly recognize the importance of correlation between the rejection region and the outputs. However, they do not study the statistical tradeoff in this paper and their rates on the required number of generated tokens is sub-optimal. 

Watermarks can also be injected with private forgeability and public verifiability~\citep{fairoze2023publicly}, hence functioning effectively as digital signatures.  Meanwhile, various attack algorithms against watermarking schemes were also studied \citep{kirchenbauer2023watermark,kirchenbauer2023reliability,sato2023embarrassingly,zhang2023watermarks,kuditipudi2023robust}. These attacking schemes apply quality-preserving perturbations to the watermarked outputs in delicate ways, and are therefore modelled by the perturbation graph (Definition~\ref{def:perturb_graph}) in the robust watermark framework in Section~\ref{sec:robust}. With the success of various attacking methods, robustness becomes an important consideration in watermarking techniques. However, \citet{zhang2023watermarks} proves that it is only feasible to achieve robustness to a well-specified set of attacks, instead of all. This fact aligns with our Theorem~\ref{thm:robust_rates}, which characterizes the fundamental limits of robust watermarking under different attacking powers. 
% \citet{piet2023mark} creates a benchmark for current statistical watermarks.

\fi
\section{Proof of \Cref{thm:crypto_stat_rates}}\label{sec:crypto_stat_rates_proof}

\begin{proof}
Let $\rho'$ denote the marginal probability of $X$ and let $\eta$ denote the marginal probability of $R$. In the bound of Type I error, choosing $\pi = \delta_y$ yields
\begin{align}\label{eq:bound_marginal_R}
    \alpha \geq &~ \PP_{X \sim \pi, R \sim \cP(\Omega,\cdot)}(X \in R) \notag\\
    = &~ \PP_{R \sim \eta}(y \in R) \notag\\
    = &~ \sum_{R \in 2^\Omega} \left(\sum_{x \in \Omega} \rho'(x) \cP(R|x) \right) \cdot \mathbbm{1}(y \in R).
\end{align}
Now notice that
\begin{align*}
    \cP(X \in R) = &~ \E_{\cP}[\mathbbm{1}(X \in R)]\\
    = &~ \sum_{y \in \Omega} \sum_{R \in 2^\Omega} \rho'(y) \cP(R|y) \mathbbm{1}(y \in R)\\
    = &~ \sum_{y \in \Omega} \underbrace{\left(\sum_{R \in 2^\Omega} \rho'(y) \cP(R|y) \cdot \mathbbm{1}(y \in R)\right)}_{A(y)}.
\end{align*}
For the term $A(y)$, we first know that $A(y) \leq \rho'(y)$. Applying Eq.~\eqref{eq:bound_marginal_R}, we further have
\begin{align*}
    A(y) \leq &~ \sum_{R \in 2^\Omega} \left(\sum_{x \in \Omega} \rho'(x) \cP(R|x) \right) \cdot \mathbbm{1}(y \in R)\\
    \leq &~ \alpha.
\end{align*}
Combining the above two inequalies, it follows that
\begin{align*}
    \cP(X \in R) \leq &~ \sum_{y \in \Omega} \left(\alpha \wedge \rho'(y) \right)\\
    = &~ 1 - \sum_{x \in \Omega: \rho'(x) > \alpha} \left(\rho'(x) - \alpha\right)\\
    \leq &~ 1- \min_{\TV(\rho'\|\rho) \leq \epsilon} \sum_{x \in \Omega: \rho'(x) > \alpha} \left(\rho'(x) - \alpha\right)\\
    \leq &~ 1 - \left(\sum_{x \in \Omega: \rho(x) > \alpha} \left(\rho(x) - \alpha\right) - \epsilon\right)_+
\end{align*}
where first equality is achieved by
\begin{align*}
    \rho' = \arg \min_{\TV(\rho'\|\rho) \leq \epsilon} \sum_{x \in \Omega: \rho'(x) > \alpha} \left(\rho'(x) - \alpha\right)
\end{align*}
and the second inequality is achieved when $\sum_{x \in \Omega: \rho(x) < \alpha} \left(\alpha - \rho(x)\right) \geq \epsilon$, a sufficient condition for which being $|\Omega| \geq 1/\alpha$. This establishes the optimal Type II error.

Finally, to verify that $\cP^*$ satisfies the conditions, the condition $\TV(\cP^*(\cdot,2^\Omega)\|\rho) \leq \epsilon$ is apparently satisfied. For any $y \in \Omega$ we have
\begin{align*}
    \PP_{R \sim \eta}(y \in R) = &~ \sum_{x \in \Omega} \rho^*(x) \cdot \PP(R = \{x\}) \cdot \mathbbm{1}(y = x)\\
    = &~ \rho^*(y) \cdot \left(1 \wedge \frac{\alpha}{\rho^*(y)}\right)\\
    \leq &~ \alpha.
\end{align*}
This implies the $\sup_{\pi \in \Delta(\Omega)}\PP_{Y \sim \pi, (X,R) \sim \cP^*}(Y \in R) \leq \alpha$ because any $\pi$ can be written as linear combination of $\delta_y$. Moreover,
\begin{align*}
    \cP^*(X \in R) = &~ \sum_{x \in \Omega} \rho^*(x) \cdot \PP(R = \{x\})\\
    = &~ \sum_{y \in \Omega} \left(\alpha \wedge \rho^*(y) \right)\\
    = &~ 1 - \sum_{x \in \Omega: \rho^*(x) > \alpha} \left(\rho^*(x) - \alpha\right).
    % \\
    % = &~ 1 - \left(\sum_{x \in \Omega: \rho(x) > \alpha} \left(\rho(x) - \alpha\right) - \epsilon\right)_+.
\end{align*}
This verifies that $\rho^*$ achieves the advertised Type II error. 
%$\left(\sum_{x \in \Omega: \rho(x) > \alpha} \left(\rho(x) - \alpha\right) - \epsilon\right)_+$.
\end{proof}

\section{Proof of \Cref{thm:iid_rates}}\label{sec:iid_rates_proof}

\begin{proof}
Throughout the proof we assume that $h < 1/4$, otherwise the bounds become trivial. 

We first prove the lower bound. For this purpose, we construct the hard instance: let $q_0 = H_b^{-1}(h)$ (take the one $\geq 1/2$) where $H_b$ is the binary entropy function defined by $H_b(x) = -x \ln x - (1-x) \ln (1-x)$, and set $\rho_0 = (1-q_0)\delta_{x_1} + q_0 \delta_{x_2}$ where $x_1,x_2$ are two different elements in $\Omega_0$. Then Lemma~\ref{lem:entropy_const_bound} implies that $q_0 \geq 3/4$. By \Cref{thm:crypto_stat_rates},
\begin{align*}
    \beta = 1-\cP(X \in R) = &~ \sum_{x \in \Omega: \rho(x) > \alpha} \left(\rho(x) - \alpha\right)\\
    \geq &~ \frac{1}{2} \cdot \PP\left(\rho(X) \geq 2 \alpha \right)\\
    = &~ \frac{1}{2} \cdot  \PP\left(\sum_{i=1}^n \ln \rho_0(X_i) \geq \ln (2\alpha )\right)\\
    \geq &~ \mathbbm{1} (n \ln q_0 \geq \ln (2\alpha) ) \cdot \frac{1}{2} q_0^n\\
    \geq &~ \mathbbm{1} (2n (1- q_0) \leq - \ln (2\alpha) ) \cdot \frac{1}{2} \exp\left(-2n(1-q_0)\right)\\
    \geq &~ \mathbbm{1} \left(n \leq \frac{\ln \frac{1}{2\alpha}}{2h /\ln \frac{\ln 2}{h}} \right) \cdot \frac{1}{2} \exp\left(-\frac{2nh}{\ln \frac{\ln 2}{h}}\right)
\end{align*}
where the last inequality follows from Lemma~\ref{lem:tight_entropy_bound}. It follows that
\begin{align}\label{eq:lower_bound_part_1}
    n(h,\alpha,\beta) \geq \frac{\ln \frac{\ln 2}{h}}{2h}\cdot \left(\ln \frac{1}{2\alpha} \wedge \ln \frac{1}{2\beta}\right).
\end{align}
Furthermore, suppose $n \leq \frac{\ln \frac{1}{2\alpha}}{4 (1-q_0)\ln \frac{1}{1-q_0}}$. Define $Y = \sum_{i=1}^n \mathbbm{1}(\rho_0(X_i) = 1-q_0)$, then notice that $Y \sim \Binom (n,1-q_0)$ and if $Y \leq \frac{\ln \frac{1}{2\alpha}}{2 \ln \frac{1}{1-q_0}}$, then 
\begin{align*}
    \sum_{i=1}^n \ln \rho_0(X_i) \geq &~ \frac{ \ln \frac{1}{2\alpha}}{2 \ln \frac{1}{1-q_0}} \cdot \ln (1-q_0) + n \cdot \ln q_0\\
    \geq &~ \ln (2\alpha )
\end{align*}
where the last inequality is due to $n \cdot \ln q_0 \geq -2(1-q_0)n \geq -2(1-q_0)\frac{\ln \frac{1}{2\alpha}}{4(1-q_0)\ln \frac{1}{1-q_0}} =  \frac{\ln (2\alpha)}{2 \ln \frac{1}{1-q_0}} \geq \frac{\ln(2\alpha)}{2}$. 
Applying this and Markov's inequality,
\begin{align*}
    \PP\left(\sum_{i=1}^n \ln \rho_0(X_i) \geq  \ln (2\alpha )\right) \geq &~ \PP\left(Y \leq \frac{2 \ln \frac{1}{2\alpha}}{\ln \frac{1}{1-q_0}}\right)\\
    \geq &~ 1- \frac{n(1-q_0)}{\frac{2 \ln \frac{1}{2\alpha}}{\ln \frac{1}{1-q_0}}}\\
    \geq &~ \frac{1}{2},
\end{align*}
A contradiction to $\PP(\rho(X) \geq 2\alpha) \leq 2\beta$. As a result,
\begin{align}\label{eq:lower_bound_part_2}
    n(h,\alpha,\beta) \geq &~ \frac{\ln \frac{1}{2\alpha}}{4(1-q_0)\ln \frac{1}{1-q_0}} \notag\\
    \geq &~ \frac{\ln \frac{1}{2\alpha}}{4h}.
\end{align}
Combining Eq.~\eqref{eq:lower_bound_part_1} and Eq.~\eqref{eq:lower_bound_part_2}, we established the lower bound.

For the upper bound, we define $q = \max_{x \in \Omega_0}\rho_0(x)$, then Lemma~\ref{lem:entropy_const_bound} implies that $q \geq 1/2$. Define $Y = \sum_{i=1}^n \mathbbm{1}(\rho_0(X_i) \neq q)$ (recall that $Y \sim \Binom (n,1-q)$). It suffices to show when
\begin{align*}
    n = 900\left(\frac{2\ln \frac{9k}{h}}{h}\cdot \left(\ln \frac{1}{\alpha} \wedge \ln \frac{1}{\beta}\right)\right) \vee \frac{(18+36\ln (9k))\ln\frac{1}{\alpha}}{h}
\end{align*}
the Type II error of the UMP watermark $1-\cP^*(X \in R) \leq \beta$.

By Theorem~\ref{thm:crypto_stat_rates} and Bennett's inequality,
\begin{align}\label{eq:concentration_bennett}
    1-\cP^*(X \in \R) = &~ \sum_{x \in \Omega: \rho(x) > \alpha} \left(\rho(x) - \alpha\right)\notag\\
    \leq &~ \PP\left(\rho(X) \geq \alpha \right)\notag\\
    = &~ \PP\left(\sum_{i=1}^n \ln \rho_0(X_i) \geq \ln (\alpha )\right)\notag\\
    \leq &~ \PP\left(Y \leq \frac{\ln \frac{1}{\alpha}}{\ln \frac{1}{1-q}} \right)\notag\\
    \leq &~ \exp\left(-nq(1-q) \theta \left(\frac{1-q-\frac{\ln \frac{1}{\alpha}}{n\ln \frac{1}{1-q}}}{q(1-q)}\right)\right)
\end{align}
where $\theta(x) = (1+x) \ln(1+x) - x$; the penultimate inequality follows from $\sum_{i=1}^n \ln \rho_0(X_i) \leq Y\ln(1-q)$. 

Notice that by Lemma~\ref{lem:entropy_const_bound},
\begin{align*}
    (1-q) \ln \frac{1}{1-q} \geq &~ \frac{h}{9\ln \frac{9k\ln (9k)}{h}} \cdot \ln\frac{\ln \frac{1}{h}}{h}\\
    = &~ h \cdot \frac{\ln \ln \frac{1}{h} + \ln \frac{1}{h}}{9\left(\ln \frac{1}{h}+\ln (9k\ln (9k))\right)}\\
    \geq &~ \frac{h}{9+\ln (9k\ln (9k))}.
\end{align*}
Since $n \geq \frac{(18 + 36\ln (9k\ln (9k)))\ln \frac{1}{\alpha}}{h}$, we have $n \geq \frac{2}{1-q}\frac{\ln \frac{1}{\alpha}}{\ln \frac{1}{1-q}}$. Under this condition, we have the simplification
\begin{align*}
    \theta \left(\frac{1-q-\frac{\ln \frac{1}{\alpha}}{n\ln \frac{1}{1-q}}}{q(1-q)}\right) \geq &~ \theta\left(\frac{1}{2q}\right)\\
    \geq &~ \frac{1}{50}.
\end{align*}
Plugging back to Eq.~\eqref{eq:concentration_bennett}, we have
\begin{align}\label{eq:bound-beta}
    1-\cP^*(X \in \R) \leq &~ \exp\left(-nq(1-q) \theta \left(\frac{1-q-\frac{\ln \frac{1}{\alpha}}{n\ln \frac{1}{1-q}}}{q(1-q)}\right)\right)\notag\\
    \leq &~ \exp\left(-\frac{n(1-q)}{100}\right)\notag\\
    \leq &~ \exp\left(-\frac{nh}{900\ln \frac{9k \ln (9k)}{h}}\right)
\end{align}
where we applied Lemma~\ref{lem:tight_entropy_bound} in the last step. 

Furthermore, we have
\begin{align}\label{eq:bound-alpha}
    1-\cP(X \in R) 
    \leq &~ \PP\left(\sum_{i=1}^n \ln \rho_0(X_i) \geq \ln (\alpha )\right)\notag\\
    \leq &~ \mathbbm{1} \left(n \leq \frac{\ln \alpha}{\ln q} \right) \notag\\
    \leq &~ \mathbbm{1} \left(n \leq 900\left(\frac{\ln \frac{9k \ln (9k)}{h}}{h}\cdot \ln \frac{1}{\alpha} \right)\right)
\end{align}
where the last step is due to Lemma~\ref{lem:tight_entropy_bound}. 
Combining Eq.~\eqref{eq:bound-beta} and Eq.~\eqref{eq:bound-alpha}, we know that $1-\cP^*(X \in \R) \leq \beta$ when $n \geq 900\left(\frac{\ln \frac{9k \ln (9k)}{h}}{h}\cdot \left(\ln \frac{1}{\alpha} \wedge \ln \frac{1}{\beta}\right)\right)$. This establishes the upper bound.
\end{proof}

\subsection{Supporting lemmata}

\begin{lemma}[\citet{topsoe2001bounds}, Theorem 1.2]\label{lem:binary_entropy_upper_bound}
Define the binary entropy function $H_b: (0,1) \to \R$ as $H_b(x) = -x \ln x - (1-x) \ln (1-x)$. Then $4x(1-x) \leq H_b(x) \leq \left(4x(1-x)\right)^{1/\ln 4}$.
\end{lemma}

\begin{lemma}\label{lem:entropy_const_bound}
Suppose $\rho$ is a probability measure over $\Omega$ such that $H(\rho) = h$, define $q = \max_{x \in \Omega}\rho(x)$. If $H(\rho) \leq 1/4$, then $q \geq 1/2$. Furthermore, if $H_b(q) \leq 1/4$, then $q \geq 3/4$.
\end{lemma}
\begin{proof}
Suppose $q \leq 1/2$. By convexity of $H$,
\begin{align*}
    H(\rho) \geq - \left \lfloor \frac{1}{q} \right \rfloor q \ln q \geq - \frac{1}{2} \ln \frac{1}{2} \geq 1/4.
\end{align*}
This is a contradiction.

Suppose $q \leq 3/4$, then Lemma~\ref{lem:binary_entropy_upper_bound} implies that
\begin{align*}
    H_b(q) \geq 4q(1-q) \geq 1/4.
\end{align*}
This is a contradiction.
\end{proof}

\begin{lemma}\label{lem:tight_entropy_bound}
Suppose $\rho$ is a probability measure over $\Omega$ such that $H(\rho) = h$ and $|\Omega| = k$. Define $q = \max_{x \in \Omega}\rho(x)$. If $q \geq 1/2$, then we have
\begin{align*}
    \frac{h}{9\ln \frac{9k\ln (9k)}{h}} \leq 1-q \leq \frac{h}{\ln \frac{\ln 2}{h}}
\end{align*}
\end{lemma}

\begin{proof}
We have
\begin{align*}
    H(\rho) \geq - (1-q) \ln (1-q) \geq (1-q) \cdot \ln 2.
\end{align*}
It follows that
\begin{align*}
    h \geq &~ - (1-q) \ln (1-q)\\
    \geq &~ (1-q) \ln \frac{\ln 2}{h}.
\end{align*}
Therefore $1-q \leq \frac{h}{\ln \frac{\ln 2}{h}}$. 

By the convexity of $H$ and $-q\ln q \leq 2(1-q)$, 
\begin{align*}
    H(\rho) \leq &~ -q \ln q - (1-q) \ln \frac{1-q}{k}\\
    \leq &~ (1-q) \ln \frac{9k}{1-q}.
\end{align*}

This means that
\begin{align}\label{eq:upper-bound-h}
    h^2 \leq &~ (1-q)^2 \left(\ln \frac{9k}{1-q}\right)^2\notag\\
    \leq &~ 2(1-q)^2 \left(\ln^2 (9k) + \ln^2(1-q)\right)\notag\\
    \leq &~ 2(1-q)^2 \ln^2 (9k) + (1-q) \cdot \left(2(1-q) \ln^2(1-q)\right)\notag\\
    \leq &~ (1-q) \cdot (\ln^2 (9k) + 18)
\end{align}
where the last inequality is due to $2(1-q) \leq 1$ and $2(1-q) \ln^2(1-q) \leq 18$. 
It follows that
\begin{align*}
    h \leq &~ (1-q) \ln \frac{9k}{1-q}\\
    \leq &~ 9 (1-q) \ln \frac{9k \ln (9k)}{h}
\end{align*}
where the last step is because $\ln\frac{1}{1-q} \leq 2\ln\frac{\ln^2(9k)+18}{h} \leq 9\ln\frac{\ln(9k)}{h}$, using Eq.~\eqref{eq:upper-bound-h}. 
This establishes $1-q \geq \frac{h}{9\ln \frac{9k\ln (9k)}{h}}$.
\end{proof}

\section{Proof of Theorem~\ref{thm:agnostic}}\label{sec:agnostic_proof}

\begin{proof}
\textbf{Lower bound.} 
Let $m = \frac{1}{\alpha}$. Notice that for any level-$\alpha$ model-agnostic watermarking $(\eta, \{\cP_\rho\}_{\rho \in \Delta(\Omega,\F)})$, the following holds
\begin{align*}
    \sum_{A \in 2^\Omega} \eta(A) \mathbbm{1}(x \in A) \leq \alpha,~\forall x \in \Omega.
\end{align*}
Furthermore, for any $\rho_0 = \mathrm{Unif}(i_1,i_2,\dots,i_m)$, we have $\beta(\cP_{\rho_0}^\ast) = 0$ and
\begin{align*}
    \beta(\cP_{\rho_0}) \geq &~ \mathbb{P}_{A \sim \eta}\left(\{i_1,\dots,i_m\} \cap A = \emptyset\right)\\
    \geq &~ \sum_{A} \eta(A) \cdot \prod_{j=1}^m \mathbbm{1}(i_j \notin A).
\end{align*}
By probabilistic method, 
\begin{align*}
    \beta(\cP_{\rho_0}) \geq &~ \max_{i_1<\cdots<i_m}\sum_{A} \eta(A) \cdot \prod_{j=1}^m \mathbbm{1}(i_j \notin A)\\
    \geq &~ \frac{1}{{n \choose m}}\sum_{i_1 < \cdots < i_m} \sum_{A} \eta(A) \cdot \prod_{j=1}^m \mathbbm{1}(i_j \notin A).
\end{align*}
It follows that the maximum Type II error loss is lower bounded by the following linear program
\begin{align*}
    v^\ast = \min_{\eta} &~ \frac{1}{{n \choose m}}\sum_{i_1 < \cdots < i_m} \sum_{A} \eta(A) \cdot \prod_{j=1}^m \mathbbm{1}(i_j \notin A)\\
    \text{s.t.} &~ \sum_{A \in 2^\Omega} \eta(A) \mathbbm{1}(x \in A) \leq \alpha,~\forall x \in \Omega,\\
    &~ \sum_{A \in 2^\Omega}\eta(A) = 1, ~\eta(A) \geq 0,~ \forall A \in 2^\Omega.
\end{align*}
By duality, this is bounded by
\begin{align*}
    &~ \min_{\eta \geq 0} \max_{\xi,\zeta \geq 0}\frac{1}{{n \choose m}}\Bigg(\sum_{i_1 < \cdots < i_m} \sum_{A} \eta(A) \cdot \prod_{j=1}^m \mathbbm{1}(i_j \notin A) + \sum_{x} \xi(x) \left( \sum_{A \in 2^\Omega} \eta(A) \mathbbm{1}(x \in A) - \alpha \right)\\
    &~ + \zeta \cdot \left(\sum_{A \in 2^\Omega}\eta(A) - 1\right)\Bigg)\\
    = &~  \max_{\xi,\zeta \geq 0} \min_{\eta \geq 0} \frac{1}{{n \choose m}}\left(\sum_{A} \eta(A) \cdot \left(\sum_{i_1 < \cdots < i_m} \prod_{j=1}^m \mathbbm{1}(i_j \notin A) + \sum_{x} \xi(x) \mathbbm{1}(x \in A) + \zeta \right) - \alpha \cdot \sum_{x}\xi(x) - \zeta\right)\\
    \geq &~ \min_{\eta \geq 0} \frac{1}{{n \choose m}}\sum_{l=1}^n\sum_{|A| = l} \eta(A) \cdot \left({n-l \choose m} + l \cdot \xi^* + \zeta^* \right) - \frac{\alpha n \xi^* + \zeta^*}{{n \choose m}}
\end{align*}
where $\xi^* = {n-\alpha n-1 \choose m-1}$ and $\zeta^* = 0$. %-\left({n-\alpha n \choose m} + \alpha n \cdot \xi^*\right) + \left({n-\alpha n \choose m} + \alpha n \cdot {n-\alpha n \choose m} \cdot \frac{m}{n-m}\right)$. 

Since $f(l) := {n-l \choose m} + l \cdot \xi^* + \zeta^*$ is a convex function and achieves the minimum at $l^* = \alpha n$, we have ${n-l \choose m} + l \cdot \xi^* + \zeta^* \geq {n-\alpha n \choose m} + \alpha n \cdot {n-\alpha n - 1 \choose m - 1}$ for all $l \in [n]$ and thus
\begin{align*}
    \textbf{RHS} \geq %- \frac{\alpha n \xi^* + \zeta^*}{{n \choose m}}
    \frac{{n-\alpha n \choose m}}{{n \choose m}} = \frac{{n - m \choose \alpha n}}{{n \choose \alpha n}} = \frac{{n - \frac{1}{\alpha} \choose \alpha n}}{{n \choose \alpha n}}.
\end{align*}

\textbf{Upper bound. } 
Notice that the marginal distribution of reject region
\begin{align*}
    \eta^*(A) = \begin{cases}
        \frac{1}{{n \choose \alpha n}}, &~ \text{if}~ |A| = \alpha n\\
        0, &~ \text{otherwise}
    \end{cases}.
\end{align*}
already guarantees Type I error $\leq \alpha$. 
It suffices to show \textbf{(*)}: for any $\rho \in \Delta(\Omega,\F)$, there exists a coupling $\cP_\rho$ of $\eta^*$ and $\rho$ such that $\mathbb{P}_{(x,A) \sim \cP_\rho}(x \notin A) \leq \frac{{n - \frac{1}{\alpha} \choose \alpha n}}{{n \choose \alpha n}} + \sum_{x: \rho(x) \geq \alpha} (\rho(x) - \alpha)$. 

Define $p$ as the projection from $\Omega \times 2^\Omega$ to $2^\Omega$, i.e. $p(V) = \{A \in 2^\Omega: \exists x \in \Omega, ~s.t. (x,A) \in V\}$. Let $W:= \{(x,A) \in \Omega \times 2^\Omega: x \in A\}$. To show the above, we check the Strassen's condition
\begin{align}\label{eq:strassen-cond}
    \rho(U) - \eta^*\left(p\left(W \cap (U \times 2^\Omega)\right)\right) \leq \frac{{n - \frac{1}{\alpha} \choose \alpha n}}{{n \choose \alpha n}} + \sum_{x: \rho(x) \geq \alpha} (\rho(x) - \alpha), &~ \forall U \subset \Omega.
\end{align}
Indeed, given Eq.~\eqref{eq:strassen-cond}, Theorem 11 in \citet{strassen1965existence} establishes \textbf{(*)}.

In the rest of the proof, we show Eq.~\eqref{eq:strassen-cond}. Fix $U$ with cardinality $k$. First notice that $\rho(U) - \sum_{x: \rho(x) \geq \alpha} (\rho(x) - \alpha) \leq (\alpha k \wedge 1) $. 
Since $p\left(W \cap (U \times 2^\Omega)\right) = \{A \in 2^\Omega: \exists i \in U, ~s.t.~ i \in A\}$, we have
\begin{align*}
    \eta^*\left(p\left(W \cap (U \times 2^\Omega)\right)\right) \geq 1 - \frac{{n - k \choose \alpha n}}{{n \choose \alpha n}} = 1 - \frac{{n - \alpha n \choose k}}{{n \choose k}}.
\end{align*}
If $k \leq \frac{1}{\alpha}$, then because $g(k) := \alpha k - 1 + \frac{{n - \alpha n \choose k}}{{n \choose k}}$ is convex and takes maximum $\frac{{n - \alpha n \choose \frac{1}{\alpha}}}{{n \choose \frac{1}{\alpha}}} = \frac{{n - \frac{1}{\alpha} \choose \alpha n}}{{n \choose \alpha n}}$ at $k^* = \frac{1}{\alpha}$, we have
\begin{align*}
    \rho(U) - \eta^*\left(p\left(W \cap (U \times 2^\Omega)\right)\right) \leq &~ \alpha k - 1 + \frac{{n - \alpha n \choose k}}{{n \choose k}} + \sum_{x: \rho(x) \geq \alpha} (\rho(x) - \alpha)\\
    = &~ \frac{{n - \frac{1}{\alpha} \choose \alpha n}}{{n \choose \alpha n}} + \sum_{x: \rho(x) \geq \alpha} (\rho(x) - \alpha).
\end{align*}
If $k \geq \frac{1}{\alpha}$, then since $\frac{{n - \alpha n \choose k}}{{n \choose k}} = \frac{{n - k \choose \alpha n}}{{n \choose \alpha n}}$ is monotonously decreasing in $k$,
\begin{align*}
    \rho(U) - \eta^*\left(p\left(W \cap (U \times 2^\Omega)\right)\right) \leq &~ \frac{{n - \alpha n \choose k}}{{n \choose k}} + \sum_{x: \rho(x) \geq \alpha} (\rho(x) - \alpha)\\
    = &~ \frac{{n - \frac{1}{\alpha} \choose \alpha n}}{{n \choose \alpha n}} + \sum_{x: \rho(x) \geq \alpha} (\rho(x) - \alpha).
\end{align*}
Combining, we establishes Eq.~\eqref{eq:strassen-cond}.

Under the condition $\alpha \to 0_+$ and $1/(\alpha n) \to 0_+$, the rate displayed in Theorem 1 simplifies to:
\begin{align*}
\frac{(n-\alpha n)(n - \alpha n - 1) \cdots (n-\alpha n - 1/\alpha +1)}{n ( n-1) \cdots (n - 1/\alpha + 1)} &\asymp (1-\alpha)^{1/\alpha} \to e^{-1}.
\end{align*}
This concludes the proof.
\end{proof}

\section{Proof of Theorem~\ref{thm:agnostic_level}}\label{sec:agnostic_level_proof}

\begin{proof}
The proof largely follows the proof of Theorem~\ref{thm:agnostic}. 
Notice that the marginal distribution of reject region
\begin{align*}
    \eta^*(A) = \begin{cases}
        \frac{1}{{n \choose \alpha n}}, &~ \text{if}~ |A| = \alpha n\\
        0, &~ \text{otherwise}
    \end{cases}.
\end{align*}
already guarantees Type I error $\leq \alpha$. 
In what remains, we define $p$ and $W$ in the same way in the proof of Theorem~\ref{thm:agnostic} and check the Strassen's condition
\begin{align}\label{eq:strassen-cond-level}
    \rho(U) - \eta^*\left(p\left(W \cap (U \times 2^\Omega)\right)\right) \leq \frac{{n - \frac{1}{\alpha} \choose \alpha n}}{{n \choose \alpha n}} + \sum_{x: \rho(x) \geq \alpha} (\rho(x) - \alpha), &~ \forall U \subset \Omega.
\end{align}

Fix $U$ with cardinality $k$. Due to the condition of $\sup_{\omega \in \Omega}\rho(\{\omega\}) \leq \kappa$, we have $\rho(U) - \sum_{x: \rho(x) \geq \alpha} (\rho(x) - \alpha) \leq (\kappa k \wedge 1) $. 
Since $p\left(W \cap (U \times 2^\Omega)\right) = \{A \in 2^\Omega: \exists i \in U, ~s.t.~ i \in A\}$, we have
\begin{align*}
    \eta^*\left(p\left(W \cap (U \times 2^\Omega)\right)\right) \geq 1 - \frac{{n - k \choose \alpha n}}{{n \choose \alpha n}} = 1 - \frac{{n - \alpha n \choose k}}{{n \choose k}}.
\end{align*}
If $k \leq \frac{1}{\kappa}$, then
\begin{align*}
    \rho(U) - \eta^*\left(p\left(W \cap (U \times 2^\Omega)\right)\right) \leq &~ \kappa k - 1 + \frac{{n - \alpha n \choose k}}{{n \choose k}} + \sum_{x: \rho(x) \geq \alpha} (\rho(x) - \alpha)\\
    = &~ \frac{{n - \frac{1}{\kappa} \choose \alpha n}}{{n \choose \alpha n}} + \sum_{x: \rho(x) \geq \alpha} (\rho(x) - \alpha).
\end{align*}
where the second step follows from the fact that $g(k) := \kappa k - 1 + \frac{{n - \alpha n \choose k}}{{n \choose k}}$ is convex and takes maximum $\frac{{n - \alpha n \choose \frac{1}{\kappa}}}{{n \choose \frac{1}{\kappa}}} = \frac{{n - \frac{1}{\kappa} \choose \alpha n}}{{n \choose \alpha n}}$ at $k^* = \frac{1}{\kappa}$.

If $k \geq \frac{1}{\kappa}$, then
\begin{align*}
    \rho(U) - \eta^*\left(p\left(W \cap (U \times 2^\Omega)\right)\right) \leq &~ \frac{{n - \alpha n \choose k}}{{n \choose k}} + \sum_{x: \rho(x) \geq \alpha} (\rho(x) - \alpha)\\
    = &~ \frac{{n - \frac{1}{\kappa} \choose \alpha n}}{{n \choose \alpha n}} + \sum_{x: \rho(x) \geq \alpha} (\rho(x) - \alpha).
\end{align*}
since $\frac{{n - \alpha n \choose k}}{{n \choose k}} = \frac{{n - k \choose \alpha n}}{{n \choose \alpha n}}$ is monotonously decreasing in $k$
Combining, we establishes Eq.~\eqref{eq:strassen-cond-level}.

Combining the above cases, we checked Strassen's condition and hence the statement follows.
\end{proof}

\section{Proof of \Cref{thm:robust_rates}}\label{sec:robust_rates_proof}

\begin{proof}
Throughout the proof we omit the subscript in the shrinkage operator $\cS$, as $G$ is fixed. 
First notice that 
\begin{align*}
    \E_{X,R \sim \cP}\left[\min_{Y \in out(X)}\mathbbm{1}(Y \in R)\right] = &~ \cP(X \in \cS(R))\\
    = &~ \sum_{y \in \Omega}\sum_{R \in 2^\Omega}\rho(y) \cP(R|y) \mathbbm{1}(y \in \cS(R)).
\end{align*}
Further, notice that $y \in in(z)$ and $y \in \cS(R)$ implies that $z \in R$, thus
\begin{align*}
    \sum_{y \in in(z)}\sum_{R \in 2^\Omega}\rho(y) \cP(R|y) \mathbbm{1}(y \in \cS(R)) \leq &~ \sum_{y \in in(z)}\sum_{R \in 2^\Omega}\rho(y) \cP(R|y)\mathbbm{1}(z \in R)\\
    \leq &~ \sum_{y \in \Omega}\sum_{R \in 2^\Omega}\rho(y) \cP(R|y)\mathbbm{1}(z \in R)\\
    = &~ \PP_{X \sim \delta_z, R \sim \cP(\Omega,\cdot)}(X \in R)\\
    \leq &~ \alpha.
\end{align*}
It follows that the optimum Type II error is lower bounded by the optimum of the following Linear Program
\begin{align}\label{eq:robust_necessary_lp}
    \min_{\cP} &~ 1-\sum_{y \in \Omega}\sum_{R \in 2^\Omega}\rho(y) \cP(R|y) \mathbbm{1}(y \in \cS(R))\\
    s.t. &~ \sum_{y \in in(z)}\sum_{R \in 2^\Omega}\rho(y) \cP(R|y) \mathbbm{1}(y \in \cS(R)) \leq \alpha, \sum_{R \in 2^\Omega}\cP(R|z) = 1, 0 \leq \cP(R|z) \leq 1, ~\forall z \in \Omega, R \in 2^\Omega.\notag
\end{align}
We claim that the minimum in Eq.~\eqref{eq:robust_necessary_lp} is equal to the minimum of Eq.~\eqref{eq:robust_lp}. 
Indeed, it suffices to show that Eq.~\eqref{eq:robust_necessary_lp} is optimized when $\cP(\cdot |y_0)$ is supported on $\left\{\emptyset, \cS^{-1}(\{y_0\})\right\}$ (then setting $x(y) \equiv \cP(\cS^{-1}(\{y\})|y)$ reduces Eq.~\eqref{eq:robust_necessary_lp} to Eq.~\eqref{eq:robust_lp}). To see this, consider any minimizer $\wt\cP$ such that there exists $y_0 \in \Omega$ and $R_0 \notin \left\{\emptyset, \cS^{-1}(\{y_0\})\right\}$, with $\wt\cP(R_0|y_0) > 0$. We will show that there exists $\bar \cP$ such that it achieves the no greater objective value, and satisfies $|\supp(\bar \cP(\cdot |y_0)) \cap \left\{\emptyset, \cS^{-1}(\{y_0\})\right\}^c| = |\supp(\wt \cP(\cdot |y_0))\cap \left\{\emptyset, \cS^{-1}(\{y_0\})\right\}^c|-1$ and $|\supp(\bar \cP(\cdot |y))| = |\supp(\wt \cP(\cdot |y))|$ for all other $y \in \Omega$. Iteratively applying this argument, we reduce $\supp(\wt \cP(\cdot |y))\cap \left\{\emptyset, \cS^{-1}(\{y\})\right\}^c$ to $\emptyset$ for any $y \in \Omega$ and thereby prove the claim.

Consider the following two cases.

\textbf{Case 1: $y_0 \notin \cS(R_0)$. } Then letting 
\begin{align*}
    \bar \cP(R|y) = \begin{cases}
    \wt\cP(R_0|y) + \wt\cP(R|y), &~ y = y_0, R = \emptyset\\
    0, &~ y = y_0, R = R_0 \\
    \wt\cP(R|y), &~ \text{o.w.},
\end{cases}
\end{align*}
we observe that 
\begin{align*}
    \sum_{y \in \Omega}\sum_{R \in 2^\Omega}\rho(y) \wt \cP(R|y) \mathbbm{1}(y \in \cS(R)) = \sum_{y \in \Omega}\sum_{R \in 2^\Omega}\rho(y) \bar \cP(R|y) \mathbbm{1}(y \in \cS(R))
\end{align*}
and $\bar \cP$ satisfies all the constraints in Eq.~\eqref{eq:robust_necessary_lp}. 
It is obvious from the construction of $\bar \cP$ that $|\supp(\bar \cP(\cdot |y_0)) \cap \left\{\emptyset, \cS^{-1}(\{y_0\})\right\}^c| = |\supp(\wt \cP(\cdot |y_0))\cap \left\{\emptyset, \cS^{-1}(\{y_0\})\right\}^c|-1$ and $|\supp(\bar \cP(\cdot |y))| = |\supp(\wt \cP(\cdot |y))|$ for all other $y \in \Omega$.

\textbf{Case 2: $y_0 \in \cS(R_0)$. } Then
letting 
\begin{align*}
    \bar \cP(R|y) = \begin{cases}
    \wt\cP(R_0|y) + \wt\cP(R|y), &~ y = y_0, R = \cS^{-1}(\{y_0\})\\
    0, &~ y = y_0, R = R_0 \\
    \wt\cP(R|y), &~ \text{o.w.}
\end{cases},
\end{align*}
we observe that 
\begin{align*}
    \sum_{y \in \Omega}\sum_{R \in 2^\Omega}\rho(y) \wt \cP(R|y) \mathbbm{1}(y \in \cS(R)) = \sum_{y \in \Omega}\sum_{R \in 2^\Omega}\rho(y) \bar \cP(R|y) \mathbbm{1}(y \in \cS(R))
\end{align*}
and $\bar \cP$ satisfies all the constraints in Eq.~\eqref{eq:robust_necessary_lp} due to $\mathbbm{1}(y \in \cS(R_0)) \geq \mathbbm{1}(y \in \cS(\{y_0\}))$ for any $y \in \Omega$. 
% Therefore it achieves the no greater objective value. 
From the construction of $\bar \cP$, we know that $|\supp(\bar \cP(\cdot |y_0)) \cap \left\{\emptyset, \cS^{-1}(\{y_0\})\right\}^c| = |\supp(\wt \cP(\cdot |y_0))\cap \left\{\emptyset, \cS^{-1}(\{y_0\})\right\}^c|-1$ and $|\supp(\bar \cP(\cdot |y))| = |\supp(\wt \cP(\cdot |y))|$ for all other $y \in \Omega$.

Combining the above cases, we established our claim. 

Finally, letting $\cP^*(\cdot |y) = x^*(y) \cdot \delta_{\cS^{-1}(\{y\})}$ for all $y \in \omega$, where $x^*$ is the solution of Eq.~\eqref{eq:robust_lp}, achieves the optimum value in Eq.~\eqref{eq:robust_lp}.
\end{proof}

\end{document}